\documentclass[11pt,a4paper]{article}
\pdfoutput=1
\usepackage{authblk}
\usepackage[left=3cm,right=3cm,top=3cm,bottom=3cm]{geometry}
\usepackage{url}
\usepackage{algorithm}
\usepackage[cmex10]{amsmath}
\usepackage{amsthm}
\usepackage{amssymb}
\usepackage{caption}
\usepackage{graphicx}

\usepackage{lipsum}
\newcommand\blfootnote[1]{%
	\begingroup
	\renewcommand\thefootnote{}\footnote{#1}%
	\addtocounter{footnote}{-1}%
	\endgroup
}
\usepackage{thmtools,blindtext}
\usepackage[font=small,skip=1pt]{caption}
\usepackage{booktabs}
\usepackage{etoolbox, siunitx}
\usepackage{enumitem}
\usepackage{color}
\usepackage{url}
\usepackage[noend]{algpseudocode}
\newcommand*{\argmin}{\operatornamewithlimits{argmin}\limits}

\newcommand\norm[1]{\left\lVert#1\right\rVert}
\newcommand\inner[1]{\left\langle#1\right\rangle}

\newcommand*{\RMSE}{\operatornamewithlimits{RMSE}\limits}
\newcommand*{\vol}{\operatornamewithlimits{Vol}\limits}
\makeatletter
\renewcommand{\ALG@beginalgorithmic}{\small}
\makeatother

\newtheorem{theorem}{Theorem}
\newtheorem{Definition}{Definition}
\newtheorem{corollary}{Corollary}
\newtheorem{lemma}{Lemma}
\newtheorem{proposition}{Proposition}

\title{\Large Efficient Online Hyperparameter Optimization for Kernel Ridge Regression with Applications to Traffic Time Series Prediction}
\date{}

\author[1]{Hongyuan Zhan\thanks{Email: hongyuan.zhan@gmail.com}}
\author[2]{Gabriel Gomes}
\author[3]{Xiaoye S. Li}
\author[1]{\\Kamesh Madduri}
\author[3]{Kesheng Wu}
\affil[1]{Penn State University, Computer Science and Engineering}
\affil[2]{UC Berkeley, PATH}
\affil[3]{Lawrence Berkeley National Laboratory, Computational Research Division}

\begin{document}
\maketitle

\begin{abstract} 
\blfootnote{This manuscript is an extended version of the paper in \cite{ZhanItsc18}.}Computational efficiency is an important consideration for deploying machine learning models for time series prediction in an online setting. Machine learning algorithms adjust model parameters automatically based on the data, but often require users to set additional parameters, known as hyperparameters. Hyperparameters can significantly impact prediction accuracy. Traffic measurements, typically collected online by sensors, are serially correlated. Moreover, the data distribution may change gradually. A typical adaptation strategy is periodically re-tuning the model hyperparameters, at the cost of computational burden. In this work, we present an efficient and principled online hyperparameter optimization algorithm for Kernel Ridge regression applied to traffic prediction problems. In tests with real traffic measurement data, our approach requires as little as one-seventh of the computation time of other tuning methods, while achieving better or similar prediction accuracy.
\end{abstract}

\section{Introduction}
Modern sensors generate large amounts of timestamped measurement data.  
These data sets are critical in a wide
range of applications including traffic flow prediction, transportation management, GPS navigation, and city planning.
Machine learning-based prediction algorithms typically adjust their parameters automatically based on the data, but also require users to set additional parameters, known as hyperparameters. For example, in a kernel-based regression model, the (ordinary) parameters are the regression weights, whereas the hyperparameters include the kernel scales and regularization constants.

These hyperparameters have a strong influence on the prediction accuracy. Often, their values are set based on past experience or through time-consuming
grid searches.  In applications where the characteristics of the data
change, such as unusual traffic pattern due to upcoming concert events, these hyperparameters have to be adjusted dynamically in order to
maintain prediction quality. In this paper, we use the term hyperparameter learning, hyperparameter optimization, and hyperparameter selection/tuning interchangeably, referring to the process of configuring the model specification before model fitting.

Existing hyperparameter optimization approaches    \cite{bengio2000gradient,seegerNips2006,foo2008efficient,bergstra2011algorithms,bergstra2012random,maclaurin2015gradient,kandasamy2015high,pedregosa2016hyperparameter,luketina16,klein2016fast,jamieson2016non,franceschi17a,li2017hyperband} are designed for offline applications where the data are split into training and validation sets, making them unsuitable for online applications. Therefore, we aim to construct online hyperparameter learning strategies.

This work was motivated by online traffic flow prediction
problem. In this context as is many others, a set of learning
algorithms are used in the traffic stream prediction engine, and the
model hyperparameters are often reset periodically.  The model
re-training is scheduled according to the operation cycle.  We summarize
this deployment protocol in Algorithm~\ref{rollingprotocol} (for 1-step-ahead prediction due to simplicity, multi-steps-ahead are similar).  Under
this protocol, operators re-select the model hyperparameters every $n$
time-steps, and retrain the model every $m$ time-steps. Note that hyperparameter tuning is much more time-consuming than model fitting. For example, the widely-used grid search strategy selects different hyperparameter configurations based on trial-and-error over the validation data $V_t$ (line 4 in Algorithm \ref{rollingprotocol}). In each trial of hyperparameters, the model needs to be re-trained and re-evaluated.

\begin{algorithm}[!t]
	\caption{A \textbf{rolling} hyperparameter tuning and model fitting protocol for deploying machine learning model for time series prediction. For simplicity, we show the 1-step-ahead prediction setting here, but multistep settings are similar.}
	\label{rollingprotocol}
	\textbf{Input}: Model $\mathcal{M}$, hyperparameter tuning interval $n$ (time-steps), model fitting interval $m$ (time-steps).\\
	\textbf{Output}: Predictions $\hat{y}_{t}$, $t =$ $0$, $1$, $2$, $\dots$, $T$.
	\begin{algorithmic}[1]
		\For{$t$ $=$ $0$ to $T$}
		\If{$\left((t~\text{mod}~n) = 0\right)$ }
		\State $V_t \gets \text{Historical Data for Model Evaluation}$ 
		\State $\lambda \gets \text{Hyperparameter Tuning}(\mathcal{M},V_t)$ \Comment{\textbf{costly}}
		\EndIf
		\If{$\left((t~\text{mod}~m) = 0\right)$}
		\State $S_t \gets \text{Historical Data for Model Training}(t)$ 
		\State $\theta^*(\lambda) \gets \text{Train Model}( \mathcal{M}, S_t, \lambda)$ 
		\EndIf 
		\State $\hat{y}_{t} \gets \text{Predict}(\mathcal{M},\theta^*(\lambda))$
		\State $\text{Observe } y_t$
		\EndFor
	\end{algorithmic}
\end{algorithm}

The implication of the high computation cost of most hyperparameter tuning methods is that traffic controllers cannot afford frequent adjustments on hyperparameters. Since the distribution of traffic flow may change gradually, keeping the hyperparameters static may result in sub-optimal performance of the prediction model. However, traffic sensors collect data at a high frequency and the data stream arrives at the control center continuously, the serial correlation of measurements suggest there is a potential for optimizing the hyperparameters in an online manner. Therefore, this work proposes an online method for hyperparameter learning motivated by the need for efficient traffic time series prediction.

Online optimization \cite{zinkevich2003online,cesa2006prediction,hazan2012tutorial,hazan17nonconvex} emerged as powerful tools to reduce the computational complexity of model fitting and provide theoretical guarantees. However, when online optimization techniques applied on streaming prediction problems, one often assumes that either the learner $\mathcal{M}$ has no hyperparameters or the hyperparameters are fixed in advance. Despite of the advances in online convex optimization, the rolling prediction scheme outlined in Algorithm \ref{rollingprotocol} is still widely used in practice since almost any learning models can be deployed in this manner. Given the justification that cost of hyperparameter search dominates cost of model learning, speeding up hyperparameter selection will be very useful in practice. 
Much of the existing work in online optimization are designed for convex objective functions, while the relationship between hyperparameters and prediction accuracy
is generally unknown and very unlikely to be convex.  Therefore, a key
challenge to address is the development of an online optimization
strategy for non-convex functions.  The major contribution of this work
is an online hyperparameter learning algorithm (called OHL) for Kernel Ridge Regression. The algorithm can also be applied to certain class of models where the objective functions satisfies some smoothness assumptions. We analyze our algorithm in non-convex regret minimization framework and prove that it achieves the optimal \textit{local regret} \cite{hazan17nonconvex} under suitable assumptions.

We make the following contributions in this paper:
\begin{itemize}
	\item We design a Multiple-Kernel Ridge Regression approach for short-term traffic time time series prediction, which aims to learn the long-term periodicity, short-term deviation and trending of traffic flows simultaneously via the combination of kernels (section \ref{sec:MKLmodel}).
	
	\item To learn the model hyperparameters effectively and efficiently, we propose an online hyperparameter learning (OHL) algorithm. Our strategy is to adaptively update the model hyperparameters with streaming data (section \ref{sec:algorithm})
	
	\item We first provide an abstraction of the OHL algorithm for a class of models where the objective function satisfies some smoothness requirements, and on which the hyper-gradients can be computed. We then analyze the algorithm under regret minimization framework and show the optimality of the algorithm in terms of \textit{local regret} (section \ref{sec:regret}).
	
	\item We tested the multiple-kernel model with the proposed OHL algorithm for traffic flow prediction on I-210 highway, and compared the performance of Multiple-Kernel Ridge Regression under other popular hyperparameter tuning methods. Our method achieves similar and sometimes better prediction accuracy compared to a state-of-art hyperparameter tuning method, while using one-seventh of the computation time (section \ref{sec:exp}). 
\end{itemize}

\section{Common Hyperparameter Tuning Algorithms}
\subsection{Grid Search}
Grid Search is the simplest and most widely used hyperparameter tuning strategy. Given a validation set $V_t$ and training set $S_t$ from the historical data, grid search enumerates a user-provided list of hyperparameter settings. For each configuration, the model is fitted on $S_t$ and evaluated on $V_t$. The configuration yields the best performance on $V_t$ is selected. Suppose there are $c$ possible choices for each hyperparameter, the cost of grid search is $O(c^{d}\cdot\texttt{cost}(\theta^*(\lambda)))$, where $\texttt{cost}(\theta^*(\lambda))$ is the cost of obtaining $\theta^*(\lambda)$. Hence the computational cost of grid search grows exponentially. When grid search is applied periodically in every $n$ steps (Algorithm \ref{rollingprotocol}), the accumulated cost of hyperparameter tuning is 
\begin{equation}
O\Big({\color{blue}\frac{T}{n}} \cdot c^{d} \cdot\texttt{cost}(\theta^*(\lambda))\Big),
\label{eq:gridbigO}
\end{equation} 
where $T$ is the total number of predictions made, $d$ is the dimension of hyperparameters. 

\subsection{Random Search}
Random Search has been shown to be effective in high dimensions despite being intuitively simple~\cite{bergstra2012random}. Given a budget of $R$ random draws per hyperparameter selection period, instead of enumerating a pre-defined list of configurations, random search trials different hyperparameters. Following the analysis in \cite{bergstra2012random}, let the volume of the hyperparameter space be $\vol(\mathcal{H})$, and let volume containing targeted hyperparameters be $\vol(\mathcal{T})$, the probability of finding a target out of $R$ random draws is:
$
	1 - \left(1- \frac{\vol(\mathcal{T})}{\vol(\mathcal{H})}\right)^R.	
$
Suppose the hyperparameters offering good predictions lie in a hyper-rectangle occupying $5\%$ of the search space \cite{bergstra2012random}, i.e., $\frac{\vol(\mathcal{T})}{\vol(\mathcal{H})} = 0.05$, the probability that at least one draw from 50 trials positioned inside the target hyper-rectangle is more than $90\%$. The accumulated computational cost of random search is 
\begin{equation}
O\Big({\color{blue}\frac{T}{n}} \cdot R\cdot\texttt{cost}(\theta^*(\lambda))\Big). 
\label{eq:randombigO}
\end{equation}

\subsection{Gradient-based hyperparameter optimization}
Gradient-based hyperparameter optimization methods for offline problems were studied in \cite{bengio2000gradient,seegerNips2006,foo2008efficient,maclaurin2015gradient,pedregosa2016hyperparameter,luketina16,franceschi17a}. In the offline setting, using a training set $S$ and a hold-out validation set $V$, one may apply gradient-based algorithm with (\ref{eq:thetaGrad}) by alternatively fitting $\theta^*(\lambda)$ on $S$ and computing the update direction of hyperparameters on $V$. When the hyper-gradient is available, \cite{pedregosa2016hyperparameter,foo2008efficient} demonstrate the superior prediction performance of gradient-based tuning. The complexity is $O\Big(I\cdot \big(\texttt{cost}(\nabla_{\lambda} f) + \texttt{cost}\left(\theta^*(\lambda)\right) \big) \Big),$ where $I$ is the number of iterations taken. In general, $I = \Omega\left(\frac{1}{\epsilon}\right)$ for non-strongly convex functions \cite{rockafellar2015convex}, where $\epsilon$ is the convergence threshold. When this approach is deployed online via the rolling protocol (Algorithm \ref{rollingprotocol}), the accumulated cost of hyperparameter tuning becomes 
\begin{equation}
O\Big({\color{blue}\frac{T}{n}} \cdot \frac{1}{\epsilon} \cdot \big(\texttt{cost}(\nabla_{\lambda} f) + \texttt{cost}\left(\theta^*(\lambda)\right) \big) \Big).
\label{eq:gradbasedbigO}
\end{equation}

\subsection{Bayesian optimization methods}
Bayesian optimization \cite{snoek2012practical} is also a popular hyperparameter tuning paradigm. It has been shown that Bayesian optimization can produce state-of-art results for tuning deep learning models. Ironically, Bayesian optimizer itself uses kernels and involves (hyper)-hyperparameters. Therefore, we exclude these approaches in the Experiments section due to the complications in applying the methods.  

Note that there is a common factor of ${\color{blue}\frac{T}{n}}$ in equation (\ref{eq:gridbigO}), (\ref{eq:randombigO}), and (\ref{eq:gradbasedbigO}) due to the periodic nature in rolling hyperparameter tuning scheduled in every $n$ steps. In section \ref{sec:algorithm}, we propose an online hyperparameter optimization algorithm which removes this factor. The theoretical performance guarantees of the algorithm will be analyzed in section \ref{sec:regret}.  

\section{Multiple-Kernel Ridge Regression}
\label{sec:MKLmodel}
Traffic flow time series is dynamic and hard to predict for a number of reasons.
Despite having an approximately AM/PM and weekday/weekend periodic pattern, the short term traffic variation from the mean can be significant. This can be due to traffic
accidents, weather, nearby events, and other factors. In addition, traffic measurements can be very noisy due to inherent uncertainties and measurement error. 

We use Multiple-Kernel Ridge Regression to simultaneously capture the periodicity pattern and short-term distortion of traffic  data. Kernel methods provide expressive tools to model the periodicity and the short-term nonlinear deviation. At each model learning step $\tau$, let $\mathbf{y}\in\mathbb{R}^N$ denote a vector collecting past $N$ data points. For each $y_t$, let $\mathbf{x}_t:= [y_s]_{s=t-p}^{t-1} \in \mathbb{R}^{p}$ be a vector of $p$ past flow observations that are used as predictor variables for $y_t$. The training set $S_{\tau}$ consists of pairs of past-present observations $\{(\mathbf{x}_t,y_t)\}_{t=\tau-N}^{\tau}$. In Kernel Ridge Regression, $\phi_{\lambda_K}: \mathbb{R}^p \rightarrow \mathbb{R}^q$ is a feature mapping from the raw observations to another feature space, indexed by hyperparameters $\lambda_K$. The Kernel Ridge Regression problem \cite{HTFbook,scholkopf2001generalized,scholkopf2002learning} finds a weight vector $\mathbf{w} \in \mathbb{R}^q$ that solves
\begin{equation}
\min_{\mathbf{w}} \sum_{t=\tau-N}^{\tau} \left(y_t - \phi_{\lambda_K}\left(\mathbf{x}_t\right)^T \mathbf{w} \right)^2 + \lambda_R \norm{\mathbf{w} }^2,
\label{eq:RKHS}
\end{equation}
where $\lambda_R > 0$ is a regularization hyperparameter to be selected, which controls the variance of estimation. By the Representer Theorem~\cite{scholkopf2001generalized,scholkopf2002learning}, there is $\theta:=[\theta_j]_{j=1}^N\in \mathbb{R}^N $, such that the optimal solution $\mathbf{w}^*$ can be written as $\mathbf{w}^*=\sum_{j=1}^{N} \theta_j\phi_{\lambda_K}\left(\mathbf{x}_{\tau+1-j}\right)$. Hence, instead of optimizing over $\mathbf{w}$, Eqn. (\ref{eq:RKHS}) can be equivalently solved by
\begin{equation}
\argmin_{\theta}  \Big(\mathbf{y} - K_{\lambda_K} \theta \Big)^T \Big(\mathbf{y} - K_{\lambda_K} \theta \Big) + \lambda_R \theta^TK_{\lambda_K} \theta,
\label{eq:KRR}
\end{equation}
where $K_{\lambda_K}\in \mathbb{R}^{N\times N}$, ${[K_{\lambda_K}]}_{ij} = [\phi(\mathbf{x}_i)^T \phi(\mathbf{x}_j)],i,j=1,\cdots,N.$ 
Therefore, instead of explicitly constructing the feature mapping $\phi_{\lambda_K}(\cdot)$, one may work directly with suitable kernels $K_{\lambda_K}(\cdot,\cdot): \mathbb{R}^p \times \mathbb{R}^p \rightarrow \mathbb{R}$.
Roughly speaking, kernel methods express the similarity between the $N$ training samples with a positive semi-definite kernel matrix
$K_{\lambda_K}\in\mathbb{R}^{N\times N}$, where $\lambda_K$ is a vector of hyperparameters that determine the kernel.

The hyperparameters of the Kernel Ridge Regression model are denoted by $\lambda:=[\lambda_K,\lambda_R]$. Throughout the paper, we use $\theta^*(\lambda)$ to denote the optimal solution of (\ref{eq:KRR}), highlighting its dependence on $\lambda$. The optimal solution $\theta^*(\lambda)$ can be written in closed-form:
\begin{equation}
\theta^*(\lambda) = \big( K_{\lambda_K} + \lambda_R I\big)^{-1}\mathbf{y}.
\label{eq:thetaclosedform}
\end{equation}
Different choices of kernels capture different aspects of the data. We model the periodicity of traffic flows as a function of time, and model the short-term deviation from strict periodicity by considering the memory effect from recent traffic. With slight abuse of notation, we also use $K_{\lambda_K}(\cdot,\cdot): \mathbb{R}^p \times \mathbb{R}^p \rightarrow \mathbb{R}$ to denote a pairwise kernel function on two data points. Let $y_t$, $y_{t'}$ be the traffic volume at time-stamp $t$, $t'$ respectively. The periodic kernel proposed by Mackay \cite{MackayGP} determines the periodicity pattern by the time difference $|t-t'|$ between two observations,
\begin{equation}
K^{\text{prd}}_{\nu,\omega}(t,t') := \exp\big( - \nu \sin^2\big(\frac{\pi |t-t'|}{\omega} \big) \big).
\label{eq:Kprd}
\end{equation}
In $K^{\text{prd}}_{\nu,\omega}$, $\omega>0$ is a hyperparameter controlling the period of recurrence, $\nu>0$ is another hyperparameter deciding the scale\footnote{most literature refer $l=\nu^{-1}$ as the length scale, we use the reciprocal for ease of differentiation later.} of ``wiggles'' in traffic flow. The short-term nonlinear effect is modelled by a squared exponential kernel using autoregressive feature $\mathbf{x}_t$,
\begin{equation}
K^{\text{se}}_{\nu}(\mathbf{x}_{t}, \mathbf{x}_{t'}):= \exp\big( - \nu \norm{\mathbf{x}_{t'} - \mathbf{x}_{t}}^2_2 \big).
\label{eq:Kse}
\end{equation} 
The kernel scale hyperparameter $\nu$ in $K^{\text{se}}_{\nu}$ has similar qualitative effects as the one in $K^{\text{prd}}_{\nu,\omega}$, but their values can be different and remain to be chosen. The automatic relevance determination (ARD) kernel is a generalization of $K^{\text{se}}_{\nu}$ allowing each component of the feature to have a different length scale,
\begin{equation}
K^{\text{ard}}_{\nu}(\mathbf{x}_{t}, \mathbf{x}_{t'}):= \exp\Big( - \sum_{i=1}^{p} \nu_i \big( y_{t-i} - y_{t'-i}  \big) \Big).
\label{eq:Kard}
\end{equation} 
The number of hyperparameters in the ARD kernel increases with the number of features, hence it is usually infeasible to optimize them with grid search. A valid kernel function gives rise to a positive semi-definite kernel matrix, where each entry is computed from the kernel function on two data points. Any linear combination between kernels produce a new one. Using this property, given $M$ different kernels, Multiple-Kernel Ridge Regression uses an composite kernel function:
\begin{equation}
K_{\lambda_K} = \beta_1 K^{(1)}_{\lambda_{K_1}} +  \beta_2 K^{(2)}_{\lambda_{K_2}} + \cdots +  \beta_M K^{(M)}_{\lambda_{K_M}},
\label{eq:ensembleKernel}
\end{equation}  
where $\sum^M_{i=1} \beta_i = 1, \beta_i \geq 0$. We consider the coefficients $\{\beta_m\}^M_{m=1}$ as hyperparameters, since they determine the final kernel matrix used in equation (\ref{eq:KRR}). (\ref{eq:ensembleKernel}) can be viewed as an ensemble learning model from different kernels \cite{wolpert1992stacked,zhou2012ensemble,Zhan18}. The composite kernel is a function of time and the autoregressive feature:
\begin{equation}
K_{\lambda_K}\Big((t,\mathbf{x}_t),(t',\mathbf{x}_{t'})\Big) = \beta_1 K^{\text{prd}}_{\nu,\omega}(t,t') + \beta_2 K^{\text{ard}}_{\nu}(\mathbf{x}_t,\mathbf{x}_{t'})
\end{equation}
After computing $\theta^*(\lambda)$ through equation~(\ref{eq:thetaclosedform}), to make a prediction for time $t$, let the vector of pairwise kernel mappings between $(t,\mathbf{x}_t)$ and $(t',\mathbf{x}_{t'})$ in the training set $S_{\tau}$ be $k_{\lambda_K}:= \big[K_{\lambda_K}\Big((t,\mathbf{x}_t),(t',\mathbf{x}_{t'})\Big)\big]_{(t',\mathbf{x}_{t'}) \in S_{\tau}} \in \mathbb{R}^N$. The prediction is given by
\begin{equation}
\hat{y}_t(\theta^*(\lambda)) =k_{\lambda_K}^T \theta^*(\lambda)
\end{equation} 
To summarize, $\lambda_K$ in Multiple-Kernel Ridge Regression includes the hyperparameters for each kernel and the kernel combination coefficients $\{\beta_m\}^M_{m=1}$. $\lambda_K$ and the regularization constant $\lambda_R$ must be set properly to balance the effects of periodicity in traffic flow, near-term nonlinear distortion due to unusual events, and estimation variance $\textendash$ resulting in a hyperparameter optimization problem.

\section{Hyperparameter Learning}
\subsection{Hyper-Gradient Computation for Kernels}
\label{sec:algorithm}
The dimension of $\lambda$ can range from tens to hundreds when an automatic relevance kernel is used with high dimensional features. Periodic hyperparameter re-selection brings heavy computational burden for an online operations. This motivates the development of online methods to adaptively learn the hyperparameters. We apply the $\ell_2$ loss function to obtain the prediction error for time-step $t$:
\begin{equation}
\ell\left(y_t,\hat{y}_t(\theta^*(\lambda))\right) = \big(y_t - k_{\lambda_K}^T \theta^*(\lambda)\big)^2 := f_t(\lambda).
\end{equation}
Notice that given $y_t$ and the training set, the prediction error is a \textit{non-convex} function of $\lambda$. Even though the loss function is convex, the nested nature of $\lambda$ in $\theta^*(\lambda)$ and the kernel $k_{\lambda_K}$ creates non-convexity. Using the chain rule, the partial derivative of $f_t(\lambda)$ with respect to the kernel hyperparameters $\lambda_K$ is:
\begin{equation}
\begin{aligned}
\frac{\partial f_t(\lambda_K)}{\partial \lambda_K} =& -2 \Big(y_t - \hat{y}_t(\theta^*(\lambda))\Big) \Big( \frac{\partial k_{\lambda_K}}{\partial \lambda_K} \theta^*(\lambda)  \Big)	\\
&  -2 \Big(y_t - \hat{y}_t(\theta^*(\lambda))\Big) \Big( k_{\lambda_K}^T \frac{\partial \theta^*(\lambda)}{\partial \lambda_K}  \Big),
\end{aligned}
\label{eq:partialK}
\end{equation} 
and the partial derivative with respect to the regularization constant $\lambda_R$ is:
\begin{equation}
\frac{\partial f_t(\lambda_K)}{\partial \lambda_R} = -2 \Big(y_t - \hat{y}_t(\theta^*(\lambda))\Big) \Big( k_{\lambda_K}^T \frac{\partial \theta^*(\lambda)}{\partial \lambda_R}  \Big).
\label{eq:partialR}
\end{equation} 
Let $\lambda_R\in[L,U]\subset \mathbb{R}^+$, since $K_{\lambda_K}$ is positive semi-definite, $A(\lambda):=\big( K_{\lambda_K} + \lambda_R I\big)$ is non-singular. Therefore, $A(\lambda)^{-1}$ is differentiable. Let $\lambda(i)$ denote the $i$-th hyperparameter. Then
\begin{equation}
\frac{\partial A^{-1}(\lambda)}{\partial \lambda(i)} = - A^{-1}(\lambda) \frac{\partial A(\lambda)}{\partial \lambda(i)} A^{-1}(\lambda) 
\end{equation} 
Consequently,
\begin{equation}
\frac{\partial \theta^*(\lambda)}{\partial \lambda(i)} = \frac{\partial A^{-1}(\lambda)}{\partial \lambda(i)} \mathbf{y}
= - A^{-1}(\lambda) \frac{\partial A(\lambda)}{\partial \lambda(i)} \theta^*(\lambda)
\label{eq:thetaGrad}
\end{equation}
Equation (\ref{eq:thetaGrad}) along with (\ref{eq:partialK}) and (\ref{eq:partialR}) provide the gradient w.r.t hyperparameters (hyper-gradient) given the loss at $y_t$. In the next subsection, we described a state-of-art gradient-based hyperparameter optimization method \cite{pedregosa2016hyperparameter}, and our rational for improving the method.

\subsection{Our Method: Online Hyperparameter Learning}
We propose an online projected hyper-gradient descent algorithm to address the computational burden of applying gradient-based hyperparameter tuning algorithms.
The idea is to compute the hyperparameter gradients $\nabla_{\lambda} f_t(\lambda)$ on-the-fly when a new datum $y_t$ is observed, then average the historical hyper-gradients to make a smoothed update on $\lambda$ before fitting $\theta^*(\lambda)$ (every $m$ steps). The entire rolling hyperparameter re-selection cycle is removed and replaced by incremental learning procedure. In addition, to speed up the computation of hyperparameter gradients, the terms in equation (\ref{eq:partialK}) and (\ref{eq:partialR}) shared with subsequent hyper-gradients are pre-computed and stored after an update.

The projected gradient update to the hyperparameters in every $m$ steps is:
\begin{equation}
\lambda^{\text{new}} = \Pi_C \Big( \lambda^{\text{old}} - \frac{\eta}{m} \sum^{\tau-1}_{t=\tau-m} \nabla_{\lambda} f_t(\lambda) \Big)
\end{equation}
where $\Pi_C(\cdot)$ is the orthogonal projection operator defined by $\Pi_C(u) = \argmin _{v \in C}\norm{u-v}^2_2$. The hyperparameter space for Multiple-Kernel Ridge Regression is $C=[U,L] \cup \Delta $, such that $\lambda(i) \in [U_i,L_i]$ if $\lambda(i)$ is not $\{\beta_j\}^M_{j=1}$, and  $\{\beta_j\}^M_{j=1}\in\Delta:= \{\boldsymbol{\beta}^T\boldsymbol{1}=1, \boldsymbol{\beta} \geq \boldsymbol{0}\}$ enforces the simplex constraints on the kernel weights. The Online Hyperparameter Learning (OHL) algorithm for Multiple-Kernel Ridge Regression is presented in Algorithm~\ref{alg:onlineHPLRidge}. 

\begin{algorithm}[t]
	\caption{Online Hyperparameter Learning (OHL) and Prediction with Multiple Kernels.}
	\label{alg:onlineHPLRidge}
	\textbf{Input}: Update window $m$, learning rate $\eta$, convex feasible set $C$, initial $\lambda_0 \in C$, number of training samples $N$, total prediction time-steps $T$.\\
	\textbf{Output}: Predictions $\hat{y}_{t}, t = 0,\ldots,T-1$.
	\begin{algorithmic}[1]
		\For{$t = 0:T-1$}
		\If{$((t\mod m) = 0)$}
		\If{$t > 0$}
		\State $\lambda_{t} = \Pi_C \Big( \lambda_{t-1} -  \eta ~ m^{-1} g_{t}  \Big)$  \Comment{{update} $\lambda$}
		\EndIf
		\State $S_t = \text{Historical Data for Model Training}(t)$ 
		\State $\theta^*(\lambda_t ) = \big( K_{\lambda_{t,K}} + \lambda_{t,R} I\big)^{-1}\mathbf{y}$ \Comment{{fit model}}
		\State $J = $ compute Jacobian matrix using Eqn. (\ref{eq:thetaGrad})
		\State $g_{t} = \mathbf{0}$
		\Else
		\State $\lambda_{t} = \lambda_{t-1}$
		\EndIf 
		\State $\hat{y}_{t} =k_{\lambda_{t,K}}(t,\mathbf{x}_t)^T \theta^*(\lambda_t)$ \Comment{{prediction}}
		\State Observe $y_t$
		\State $\nabla_{\lambda} f_t(\lambda_t) = $ compute hyperparameter gradient using Eqn. (\ref{eq:partialK}), Eqn. (\ref{eq:partialR}) and pre-computed $J$
		\State $g_{t+1} = g_{t} + \nabla_{\lambda} f_t(\lambda_t)$  \Comment{hyper-gradient}
		\EndFor
	\end{algorithmic}
\end{algorithm}

\begin{algorithm}[t]
	\caption{Online Projected Gradient Descent with Lazy Updates.}
	\label{alg:OLnonconvex}
	\textbf{Input}: Update window $m$, learning rate $\eta$, convex feasible set $C$, initial $z_0\in C$, timesteps $T$.\\
	\textbf{Output}: Iterates $z_{t}, t = 0,\ldots,T-1.$
	\begin{algorithmic}[1]
		\For{$t = 0:T-1$}
		\If{$\mod(t,m) = 0$}
		\If{$t>0$}
		\State $z_t = \Pi_C \Big( z_{t-1} -  \eta ~ m^{-1} g_{t}  \Big)$
		\EndIf	
		\State $g_{t} = \mathbf{0}$
		\Else
		\State $z_t = z_{t-1}$										
		\EndIf 
		\State \texttt{Submit} $z_t$ 
		\State $\texttt{Observe cost function } f_t: C \rightarrow \mathbb{R} $
		\State \texttt{Compute} $\nabla_{z} f_t(z_t)$
		\State $g_{t+1} = g_{t} + \nabla_{z} f_t(z_t)$
		\EndFor
	\end{algorithmic}
\end{algorithm}

\subsection{Complexity}
Lines 4, 7, 14, and 15 in Algorithm~\ref{alg:onlineHPLRidge} are used for adjusting hyperparameters. Line 6  for computing $\theta^*(\lambda)$  is a common step for all rollingly-trained Kernel Ridge methods, and thus not additionally introduced by Algorithm~\ref{alg:onlineHPLRidge}. The cost of computing the Jacobian matrix for a hyperparameter in Line 7 via eqn (\ref{eq:thetaGrad}) is $O(N^2)$ due to matrix-vector multiplications, since the inverse kernel design matrix $A^{-1}(\lambda)$ has been obtained in computing $\theta^*(\lambda)$. Also, this cost only occurs in every $m$ steps. The cost of computing the partial derivative for each hyperparameter in Line 14 via eqn (\ref{eq:partialK}) and (\ref{eq:partialK}) reduces to a $O(N)$ inner product operation with a column in the pre-computed Jacobian matrix. Thus, Algorithm~\ref{alg:onlineHPLRidge} efficiently computes the hyperparameter gradients online. Projection onto simplex $\Pi_{\Delta}(\cdot)$ for the kernel coefficients $\{\beta\}^M_{i=1}$ can be computed in $O(M\log M)$ time \cite{wang2013projection}, and projection onto box constraints is a linear time operation on the number of hyperparameters. Hence the update step in line 4 can also be done efficiently. Therefore, the complexity of OHL applied on Multiple-Kernel Ridge Regression is 
\begin{equation}
O\Big( {\color{blue}\frac{T}{m} } \cdot \big( N^{\textcolor{blue}{2}} d  + M\log M \big)  + T N d \Big)
\end{equation}

\section{Theoretical Analysis with Local Regret}
\label{sec:regret}
We now present the theoretical analysis under online learning framework for non-convex functions. Algorithm~\ref{alg:OLnonconvex} is an abstraction of Algorithm~\ref{alg:onlineHPLRidge} for general non-convex function $f_t$, where the subscript $t$ represents the the time-varying nature of the hyperparameter optimization problem due to dependence on the rolling training set $S_t$. Note that since Algorithm~\ref{alg:onlineHPLRidge} is a specific implementation of Algorithm~\ref{alg:OLnonconvex}, the result extends to our hyperparameter learning problem. Online learning models the iterates $\{z_t\}^{T-1}_{t=0}$ and the functions $\{f_t\}^{T-1}_{t=0}$ as a repeated game of $T$ rounds. At each time $t$, the learner selects an iterate $z_t \in C$, where $C \subset \mathbb{R}^n$ is a compact convex set. After $z_t$ has been chosen, a cost function $f_t: C \rightarrow \mathbb{R}$ is revealed to the learner and the learner suffers a loss $f_t(z_t)$. We make the following assumptions on the cost function $f_t$. These assumptions are satisfied for kernel method hyperparameter learning problem, i.e., when $f_t(\cdot)=\ell(y_t,\hat{y}_t(\theta^*(\cdot)))$. Let $\norm{\cdot}$ to denote the Euclidean norm throughout the rest of the paper.

\begin{itemize}
	\item[\textbf{A1}.] $\sup_{z\in C} |f_t(z)| \leq M$ for all $t$.
	\item[\textbf{A2}.] $f_t$ is $L$-Lipschitz: $|f_t(z)-f_t(v)| \leq L \norm{z-v}$.
	\item[\textbf{A3}.] $f_t$ has $Q$-Lipschitz gradient:  $$\norm{\nabla f_t(z)-\nabla  f_t(v)}_2 \leq Q \norm{z-v}.$$
\end{itemize}
The performance of $\{z_t\}^{T-1}_{t=0}$ with respect to $\{f_t\}^{T-1}_{t=0}$ is studied by the measure of regret.  For \textit{convex} cost functions, the regret is typically defined by 
$$
{\sum_{t=0}^{T-1} f_t(z_t) - \min_{z\in C} \sum_{t=0}^{T-1} f_t(z)}
,$$
which is the difference between the choices $\{z\}^{T-1}_{t=0}$ and the best fixed decision in hindsight \cite{zinkevich2003online,hazan2012tutorial}. However, when the cost functions are non-convex, searching for global minimum is NP-hard in general even in the offline case where a static $f: C\rightarrow \mathbb{R}$ is known in advance. Furthermore, due to the convex constraint $z \in C$, a large number of gradient evaluations are required to discover a stationary point \cite{hazan17nonconvex}. Thus, for offline problems, a relaxed criterion is to minimize the $(C,\eta)$-\textit{projected gradient} \cite{ghadimi2016mini,hazan17nonconvex}:
\begin{equation}
{
	P(z,\nabla f(z),\eta) := \frac{1}{\eta}\Big( z - \Pi_C \big( z - \eta \nabla f(z) \big) \Big) .
}
\end{equation}
It is easy to see that $P(z,\nabla f(z),\eta)$ mimics the role of gradient in a projected gradient update:
\begin{equation}
{
	z_{t+1}:=\Pi_C \big( z_t - \eta \nabla f(z_t) \big) = z_t - \eta P(z_t,\nabla f(z_t),\eta).
}
\end{equation}
Since $C\subset \mathbb{R}^n$ is compact and convex, and assuming $f$ satisfies \textbf{A1-3}, then there exists a point $z^*\in C$ such that $P(z^*,\nabla f(z^*),\eta)=0$ \cite{hazan17nonconvex}. Therefore, as a natural extension from the offline criterion of vanishing projected gradients, the \textit{local regret} for non-convex online learning is defined as follows.
\begin{Definition}
	The local regret \cite{hazan17nonconvex} for loss functions $\{f_t\}^{T-1}_{t=0}$ and sequence of iterates $\{z_t\}^{T-1}_{t=0}$ is
	\begin{equation}
	\label{eq:defregret}
	{
		\mathcal{R}_T=\sum^{T-1}_{t=0}  \norm{P(z_t,\nabla f_t(z_t),\eta)}^2.
	}
	\end{equation}
	\label{def:localregret}
\end{Definition}
\vspace{-16pt}
Definition \ref{def:localregret} was first used by Hazan et al.\cite{hazan17nonconvex}. The following theorem shows the optimal local regret is lower bounded by $\Omega(T)$.

\begin{theorem} 
	Define $C =[-1,1]$. For any $T\geq 1$ and $\eta \leq 1$, there exists a distribution $\mathcal{D}$ of loss functions $\{f_t\}^{T-1}_{t=0}$ satisfying assumption \textbf{A1-3}, such that for any online algorthms, the local regret satisfies
	\begin{equation}
	\mathbb{E}_{\mathcal{D}}\big( \mathcal{R}_T \big) \geq \Omega(T).
	\end{equation}
	\label{thm:RegLowerBound}
\end{theorem}
\vspace{-20pt}
\begin{proof}
	See Theorem 2.7 in Hazan et al. \cite{hazan17nonconvex}.
\end{proof}
A time-smoothed follow-the-leader (FTL) algorithm was proposed in \cite{hazan17nonconvex}, achieving the optimal local regret bound $O(T)$ for non-convex functions. This algorithm computes the gradients $\{\nabla f_{t-i}(z_t)\}_{i=1}^m$ and updates the iterate $z_t$ in every step. For the hyperparameter learning problem considered in this paper, when $z$ represents hyperparameter $\lambda$, a change from $\lambda_t$ to $\lambda_{t+1}$ will require model refitting to update $\theta^*(\cdot)$ from $\theta^*(\lambda_t)$ to $\theta^*(\lambda_{t+1})$ in every step. Besides, the historical gradients are not reused in the time-smoothed FTL algorithm \cite{hazan17nonconvex}, since $\nabla f_{t-i}(\cdot)$ is re-evaluated at latest $z_t$ in every step. Hence, the method in \cite{hazan17nonconvex} becomes impractical for online hyperparameter learning given a computational budget. In contrast, Algorithm~\ref{alg:OLnonconvex} accumulates the gradients and produces an update every $m$-steps, which dramatically reduces the amount of gradient computation and model-refitting on $\theta^*(\lambda)$. As a price paid for the speed-up, we need the following additional assumption characterizing the variation of cost functions to achieve optimal regret bounds.  Let $[r]$ denote $[0,r] \cap \mathbb{Z}$.
\begin{itemize}
	\item[\textbf{A4}.] Assume there is a constant $w \in \mathbb{Z}$ \textbf{independent of} $T$, for all $m \in [w] \backslash \{0\}$, there is a constant $V_m\in\mathbb{R}^+$, such that for any $t$, the variation of gradients from the average within $m$ steps is bounded: 
	\vspace{-10pt}
\end{itemize}
\begin{equation}
\label{eq:QuardaticVarBound}
\begin{aligned}
\sup_{z\in C } \sum_{i=0}^{m-1} &\norm{ \nabla f_{t+i}(z) - \nabla F_{t,m}(z)}^2 \leq V_m,\\
\text{where  } ~~	&F_{t,m}(z) := \frac{1}{m} \sum_{i=0}^m f_{t+i}(z).
\end{aligned} 
\end{equation}

In the convex setting, variations defined similar to (\ref{eq:QuardaticVarBound}) have also been studied \cite{hazan2008extracting,hazan2009stochastic,chiang2012online}. However, the variation used in \cite{hazan2008extracting} defines $m=T$, whereas in (\ref{eq:QuardaticVarBound}) $m$ is a constant independent of $T$. Note that the quadratic variation in (\ref{eq:QuardaticVarBound}) also implies 
\begin{equation}
{
	\sup_{z\in C } \sum_{i=0}^{m-1} \norm{ \nabla f_{t+i}(z) - \nabla F_{t,m}(z)} \leq \sqrt{mV_m} .
}
\label{eq:AbsVarBound}
\end{equation}

\begin{corollary}
	There exists a distribution $\mathcal{D}$ of loss functions satisfying assumption \textbf{A1-3}, and \textbf{A4}, such that for any iterates $\{z_t\}^T_{t=1}$, the lower bound $\mathbb{E}_{\mathcal{D}}\big( \mathcal{R}_T \big) \geq \Omega(T)$ still applies. 
\end{corollary}
\begin{proof}
	See Theorem 2.7 in Hazan et al. \cite{hazan17nonconvex}, construction of $\mathcal{D}$ in Theorem \ref{thm:RegLowerBound} also satisfies assumption \textbf{A4}. 
\end{proof}

We need a few properties of projected gradients before establishing the regret bound for algorithm \ref{alg:OLnonconvex}.
\begin{lemma}
	For any $z \in C$, $\nabla f(z)$, and $\nabla g(z)$,
	\begin{equation}
	{
		\norm{P(z,\nabla f(z),\eta) - P(z,\nabla g(z), \eta)}_2 \leq \norm{\nabla f(z)-\nabla g(z)}_2  
	}
	\end{equation}
	\label{lemma:ProjectionDistance}
\end{lemma}

\begin{proof}
	An application of Lemma 2 in Ghadimi et al. \cite{ghadimi2016mini}. 	
\end{proof}

\begin{lemma}
	For any $z \in C$ and $\nabla f(z)$,
	\begin{equation}
	{
		\langle \nabla f(z) , P(z,\nabla f(z) ,\eta)  \rangle \geq \norm{P(z,\nabla f(z) ,\eta)}^2_2.
	}
	\end{equation}
	\label{lemma:ProjectionInnerProdLowerBound}
\end{lemma}
\begin{proof}
	See Lemma 1 in \cite{ghadimi2016mini} and Lemma 3.2 in \cite{hazan17nonconvex}.
\end{proof}

We now bound the local regret $\mathcal{R}_T$ of the whole sequence by the projected gradients of its subsequence. Recall that Algorithm~\ref{alg:OLnonconvex} updates the iterate $z_t$ every $m$ steps. Without loss of generality, assume $s=T/m \in \mathbb{Z}$. Let $\tau_j, j=1,\cdots,s$ denote the steps at which an increment will occur, i.e., $0,\cdots,T-1$ can be represented as 
$$\tau_0,\tau_0 + 1,\cdots,\tau_{0}+m-1,\tau_1,\cdots,\tau_{s},\tau_{s}+1,\cdots,\tau_{s}+m-1.$$
Moreover, $z_{\tau_j} = z_{\tau_j+i}$ for any $j\in[s]$ and $i\in[m-1]$.

\begin{proposition}
	\label{thm:RegretTransform}
	Let $\{\tau_j\}^s_{j=0}$ denote the steps at which an increment to the iterates will occur in Algorithm~\ref{alg:OLnonconvex}. Suppose $m$ in Algorithm~\ref{alg:OLnonconvex} is chosen such that $m\leq w$ in assumption \textbf{A4}, the local regret satisfies
	\begin{equation}
	\label{eq:RegretTransformBound}
	\begin{aligned}
	\mathcal{R}_T \leq& \sum_{j=0}^s \norm{P(z_{\tau_j},\nabla F_{\tau_j,m}(z_{\tau_j}),\eta)}^2 +    \\
	+&2\sqrt{mV_m}\sum_{j=0}^s \norm{P(z_{\tau_j},\nabla F_{\tau_j,m}(z_{\tau_j}),\eta)}  + (s+1)V_m . \\
	\end{aligned}
	\end{equation} 
\end{proposition}

\begin{proof}
	Recall $F_{t,m}(z)= m^{-1} \sum^{m-1}_{i=0} f_{t+i}(z)$ from equation (\ref{eq:QuardaticVarBound}),
	\begin{equation}
	{
		\begin{aligned}
		\mathcal{R}_T &=\sum^s_{j=0} \sum^{m-1}_{i=0} \norm{P(z_{\tau_j + i},\nabla f_{\tau_j + i}(z_{\tau_j + i}),\eta)}^2 \\
		&=\sum^s_{j=0} \sum^{m-1}_{i=0}  \Big( \big\lVert P(z_{\tau_j + i},\nabla f_{\tau_j + i}(z_{\tau_j + i}),\eta)- P(z_{\tau_j},\nabla F_{\tau_j,m}(z_{\tau_j}),\eta) + P(z_{\tau_j},\nabla F_{\tau_j,m}(z_{\tau_j}),\eta)		 \big\rVert   \Big)^2 \\
		&\leq \sum^s_{j=0} \sum^{m-1}_{i=0} \Big( \big\lVert P(z_{\tau_j},\nabla f_{\tau_j + i}(z_{\tau_j}),\eta) - P(z_{\tau_j},\nabla F_{\tau_j,m}(z_{\tau_j}),\eta) \big\rVert   
		+ \norm{P(z_{\tau_j},\nabla F_{\tau_j,m}(z_{\tau_j}),\eta)	}   \Big)^2 \\
		\end{aligned}	
	}
	\end{equation}
	where the last line follows from $z_{\tau_j+i} = z_{\tau_j}$ for $i\in[m-1]$ and triangle inequality. From Lemma \ref{lemma:ProjectionDistance},
	\begin{equation}
	{
		\begin{aligned}
		\mathcal{R}_T &\leq \sum^s_{j=0} \sum^{m-1}_{i=0} \Big(  \norm{\nabla f_{\tau_j + i}(z_{\tau_j}) - \nabla F_{\tau_j,m}(z_{\tau_j}) } +  \norm{P(z_{\tau_j},\nabla F_{\tau_j,m}(z_{\tau_j}),\eta)}  \Big)^2 \\
		&\leq \sum^s_{j=0} \sum^{m-1}_{i=0} \Big(\norm{\nabla f_{\tau_j + i}(z_{\tau_j}) - \nabla F_{\tau_j,m}(z_{\tau_j})}^2 +\norm{P(z_{\tau_j},\nabla F_{\tau_j,m}(z_{\tau_j}),\eta)}^2 \\
		&~~~~~~~~~~~~~~~~ +2 \big\lVert \nabla f_{\tau_j + i}(z_{\tau_j}) - \nabla F_{\tau_j,m}(z_{\tau_j}) \big\rVert \norm{P(z_{\tau_j},\nabla F_{\tau_j,m}(z_{\tau_j}),\eta)} \Big)
		\end{aligned}
	}
	\end{equation}
	Applying assumption \textbf{A4} and its implication (\ref{eq:AbsVarBound}) yields the claim. $\qedsymbol$
\end{proof}

Proposition \ref{thm:RegretTransform} has a very intuitive meaning: when the variation of gradients of the loss functions are bounded, the local regret is bounded by the projected gradients of loss at the updating steps $\tau_j, j\in [s]$. We now state the main theorem on the asymptotic growth of local regret in Algorithm~\ref{alg:OLnonconvex}.

\begin{theorem}
	\label{thm:RegretUpperBound}
	Let $w$ be the constant in assumption \textbf{A4}, choosing the update period $m \leq w$, learning rate $\eta \in (0,\frac{2}{Q})$ in Algorithm~\ref{alg:OLnonconvex}, the local regret satisfies
	\begin{equation}
	\mathcal{R}_T  \leq O(T).
	\end{equation}
\end{theorem}

\begin{proof}
	From assumption \textbf{A3}, $\nabla f_t(z)$ is $Q$-Lipschitz, therefore $\nabla F_{t,m}(z)$ is also $Q$-Lipschitz for all $t,m$. For any $\tau_j$,
	
	\begin{equation}
		\begin{aligned}
		F_{\tau_j,m}(z_{\tau_{j+1}}) 
		& \leq F_{\tau_j,m}(z_{\tau_j}) - \eta \inner{ \nabla F_{\tau_j,m}(z_{\tau_j}) , P\big(z_{\tau_j},\nabla F_{\tau_j,m}(z_{\tau_j}),\eta \big) } + \frac{Q\eta^2}{2} \norm{P(z_{\tau_j},\nabla F_{\tau_j,m}(z_{\tau_j}) ,\eta)}^2_2
		\end{aligned} 
	\end{equation}
	Applying Lemma \ref{lemma:ProjectionInnerProdLowerBound},
	\begin{equation}
		\begin{aligned}
		F_{\tau_j,m}(z_{\tau_{j+1}}) &\leq F_{\tau_j,m}(z_{\tau_j}) - \eta \norm{P(z_{\tau_j},\nabla F_{\tau_j,m}(z_{\tau_j}),\eta)}^2_2 + \frac{Q\eta^2}{2} \norm{P(z_{\tau_j},\nabla F_{\tau_j,m}(z_{\tau_j}),\eta)}^2_2
		\end{aligned} 
	\end{equation}
	Rearrange the terms, 
	\begin{equation}
	{
		\begin{aligned}
		\norm{P(z_{\tau_j},\nabla F_{\tau_j,m}(z_{\tau_j}),\eta)}^2_2 &\leq \frac{\Big[  F_{\tau_j,m}(z_{\tau_j}) - F_{\tau_j,m}(z_{\tau_{j+1}}) \Big]}{(\eta - \frac{Q\eta^2}{2}) } 
		\end{aligned} 
	}
	\end{equation}
	From assumption \textbf{A2}, $F_{\tau_j,m}(z)$ is $L$-Lipschitz. Since $C\subset \mathbb{R}^n$ is compact, let $D:= \max_{u,v \in C}\norm{u-v}_2$ denote the diameter of $C$. We have 
	\begin{equation}
	\label{eq:pgSquareBound}
	{
		\norm{P(z_{\tau_j},\nabla F_{\tau_j,m}(z_{\tau_j}),\eta)}^2_2 \leq LD(\eta - \frac{Q\eta^2}{2})^{-1} 
	}
	\end{equation}
	\begin{equation}
	\label{eq:pgNormBound}
	{
		\norm{P(z_{\tau_j},\nabla F_{\tau_j,m}(z_{\tau_j}),\eta)}_2 \leq L^{1/2}D^{1/2}(\eta - \frac{Q\eta^2}{2})^{-1/2} 
	}
	\end{equation}
	Summing up all $j=0$ to $s$, plugging the bounds (\ref{eq:pgSquareBound}) and (\ref{eq:pgNormBound}) into Proposition \ref{thm:RegretTransform}, and from assumption $\textbf{A4}$,
	\begin{equation}
	\begin{aligned}
	\mathcal{R}_T &\leq (T/m) LD(\eta - \frac{Q\eta^2}{2})^{-1}  + (2T\sqrt{V_m}/\sqrt{m}) L^{1/2}D^{1/2}(\eta - \frac{Q\eta^2}{2})^{-1/2} + TV_m/m  \\
	& \leq O(T). 
	\end{aligned}
	\end{equation}
	Hence the regret bound is proved as claimed. $\qedsymbol$
\end{proof}

\section{Experiments}
\label{sec:exp}
We conduct experiments to evaluate the proposed online hyperparameter learning method on synthetic and real data. On the real data, we perform 15-minutes-ahead traffic flow prediction on 13 randomly-selected sensors of different types, distributed along the I-210~\cite{I210ICM} highway in California. The primary goals of the experiments are to

\begin{enumerate}
	\item Test whether our proposed Online Hyperparameter Learning (OHL) method can \emph{adaptively learn} the hyperparameters, given a misspecified starting point. 
	
	\item Examine the \emph{computational efficiency} of OHL against other model tuning methods.
	
	\item Compare the traffic flow \emph{prediction accuracy} of OHL with other hyperparameter tuning strategies with multiple kernel models. 
\end{enumerate}

\subsection{Synthetic data}
We generate a function with periodicity and linear trends using the following scheme:
\begin{equation}
\begin{aligned}
y(t) = 1 + c_1 \text{AR}(20)  + c_2 \sin(t/\omega)
\end{aligned}
\end{equation}
Here, $c_1 = c_2 = 0.5$, $\omega=5$, the autoregressive coefficients are $\boldsymbol{\alpha}_i/2\norm{\boldsymbol{\alpha}}$ with $\boldsymbol{\alpha}_i=i$. We compare the ground truth with one-step-ahead predictions produced by kernel methods under fixed hyperparameters and under our OHL algorithm. In both cases, the initial hyperparameters are the same. We set $m=10$ and $\eta=0.001$ in Algorithm~\ref{alg:onlineHPLRidge}.

We test two kernels, a squared exponential kernel and the combination of a linear kernel and a periodic kernel. The initial kernel scale is $\nu=0.1$ for the square exponential kernel. When this choice of kernel scale is fixed and used for prediction, the prediction produces a zigzagging line, indicating the kernel scale is misspecified (green line in Figure~\ref{fig:syn_exp}). With OHL, the predictions overlap with the fixed hyperparameter case initially, given the same starting hyperparameters, but aligns much closer to ground truth after 400 time-steps. OHL is tested for multiple kernel learning using periodicity and linear kernel. The initial hyperparameters are $\beta_{\text{prd}}=1$, period $\omega=5$ and scale $\nu=10$ for the periodicity kernel, and $\beta_{\text{lin}}=0$ for the linear kernel. Hence, we expect in the fixed hyperparameter case, periodicity can be reproduced but linear trend will be hard to capture. As seen in the right chart of Figure~\ref{fig:syn_exp}, there is an equidistant gap between the ground truth and the predictions under fixed hyperparameters. With OHL, the learner soon discovers the autoregressive drift term and the predictions are almost perfectly aligned with ground truth (right chart of Figure~\ref{fig:syn_exp}). The synthetic experiments demonstrate that OHL can perform adaptive learning, which allows the users to initiate the system without the costly tuning process.
\begin{figure}[!t]
	\centering
	\includegraphics[width=0.47\textwidth]{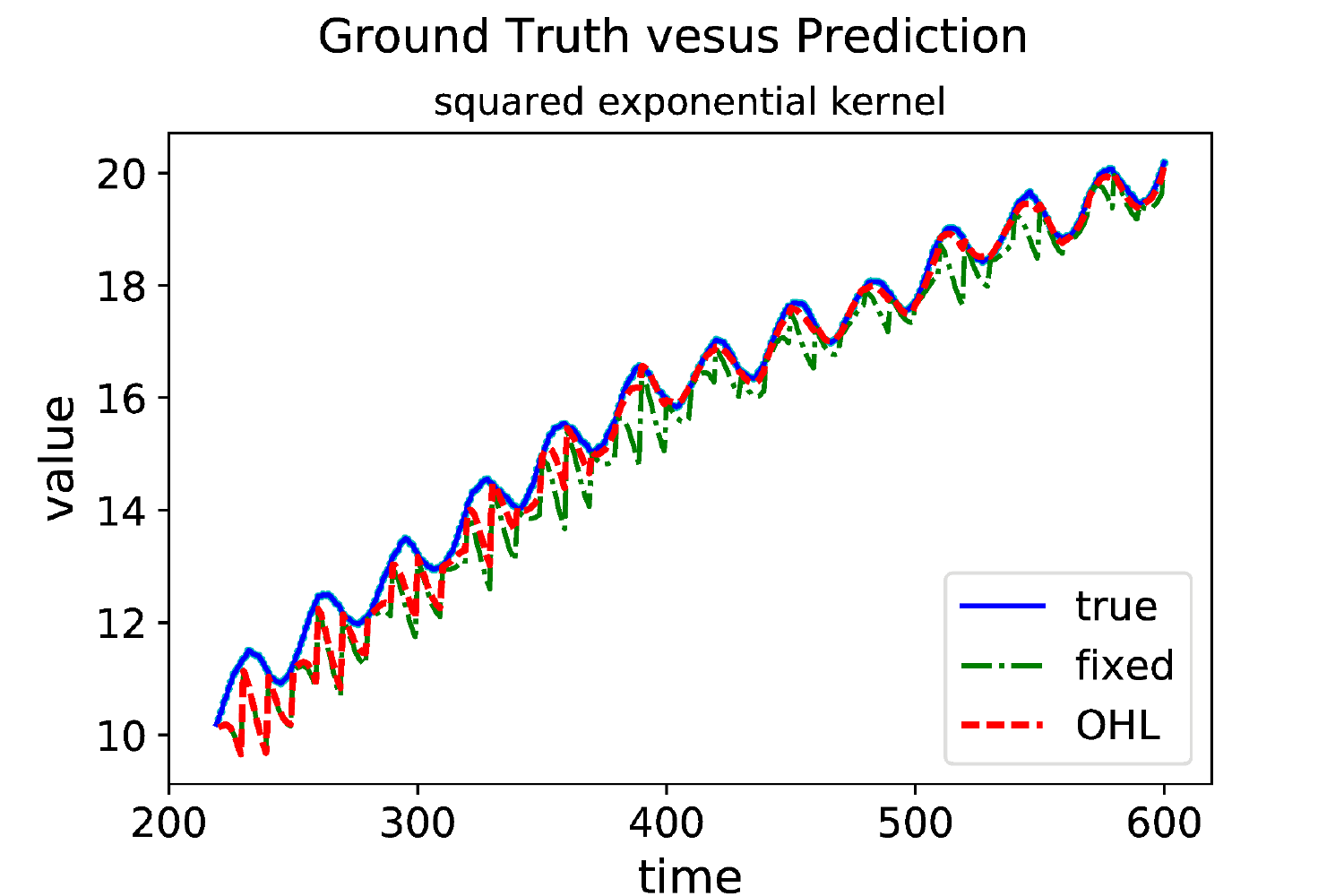}
	\includegraphics[width=0.47\textwidth]{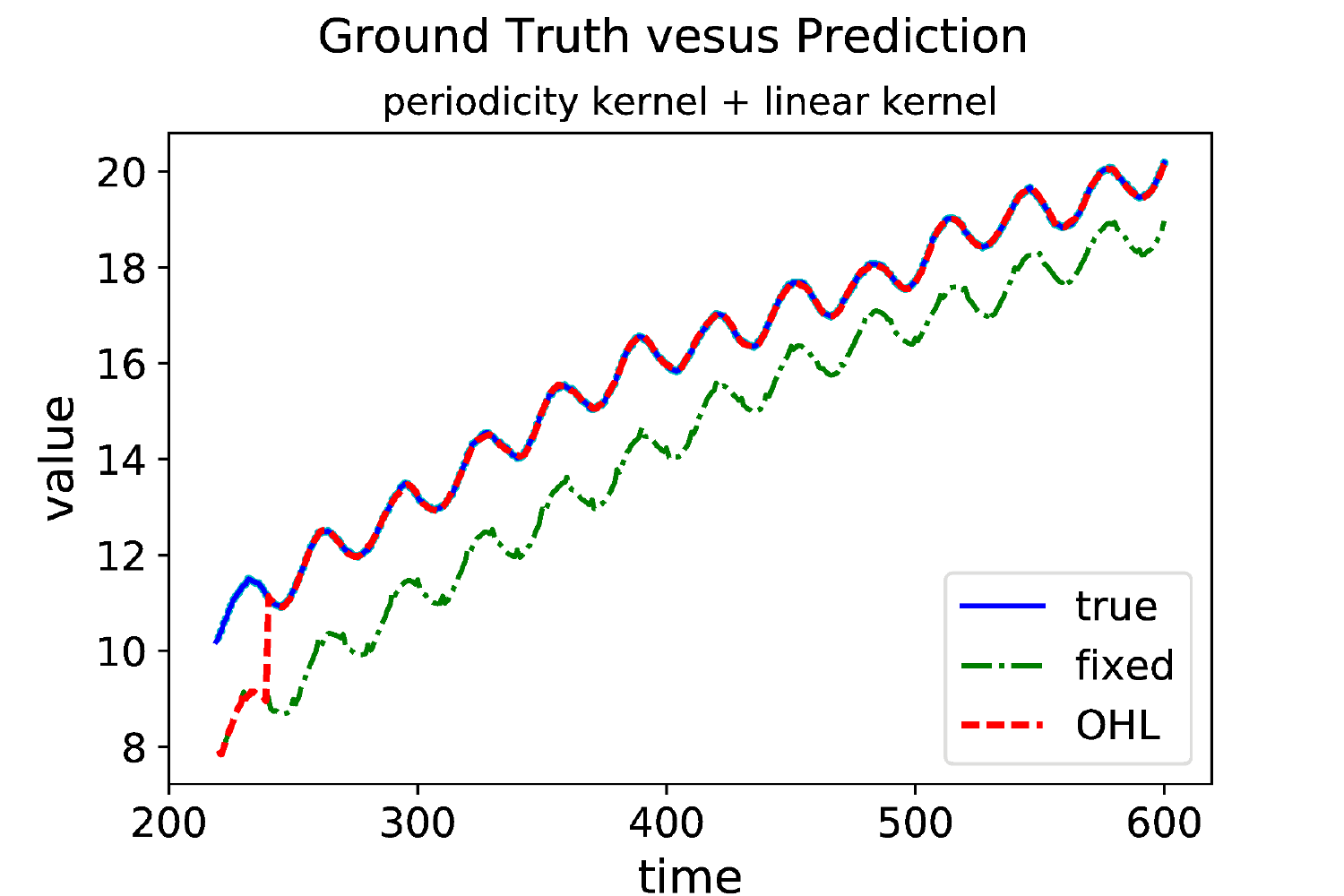}
	\caption{\small{Comparing OHL and FIXED on synthetic data. Left: squared exponential kernel, Right: Combination of periodicity and linear kernel. OHL (red line) adaptively learns towards ground truth, even when the initial hyperparameters are mis-specified.}}
	\label{fig:syn_exp}
\end{figure}

\subsection{I-210 Traffic Data}

\subsubsection{Data and Setup}
The I-210 highway is a vital route in the San Gabriel Valley region of the Los Angeles metropolitan area~\cite{I210ICM}. We use sensor data from thirteen randomly selected  locations covering both mainline detectors and ramp detectors. Measurements from January 1 to May 16 of 2017 are used. The raw data were binned using a 15 minute time window. Hence, there are 96 observations per day. 15-minutes-ahead predictions are tested. 

\begin{figure*}[ht!]
	\centering
	\includegraphics{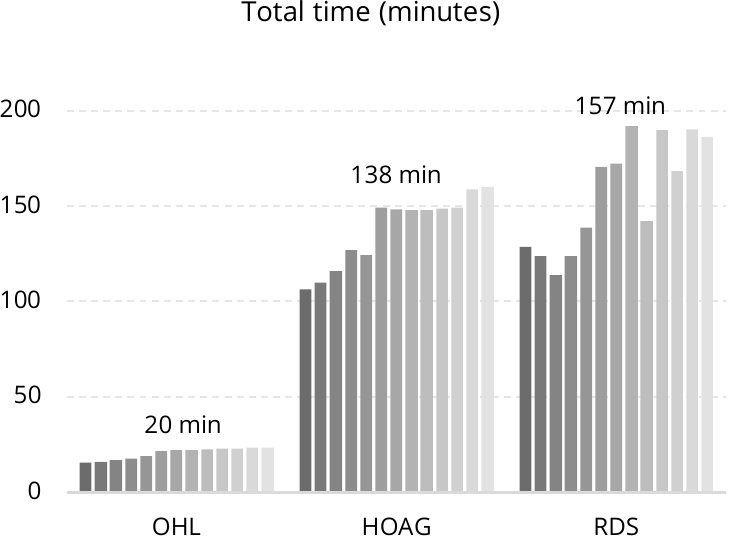}~~~~
	\includegraphics{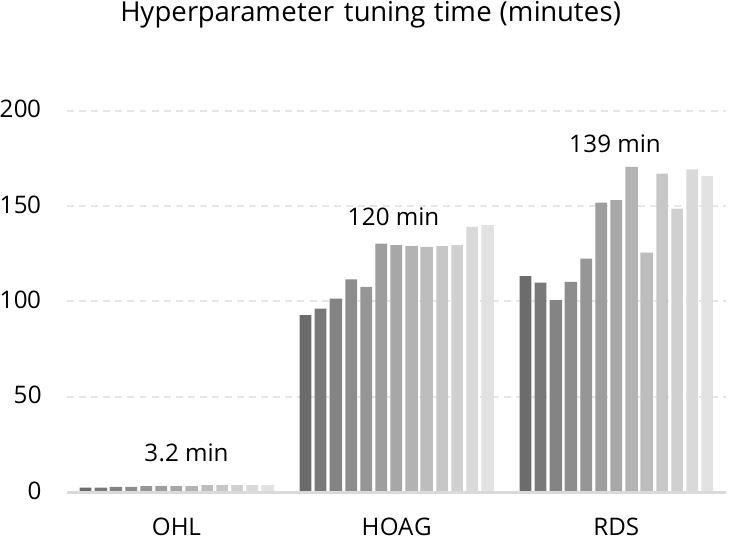}
	\caption{\small{Comparing total time and hyperparameter tuning time for OHL, HOAG, and RDS. Each bar indicates the time taken for a detector, and the labels above bars indicate average time for 13 detectors. Detectors are ordered by total time using OHL.}}
	\label{fig:runtime}
\end{figure*}

\setlength{\parskip}{0cm}
\setlength{\parindent}{1em}
We compare the computational efficiency and prediction accuracy of hyperparameter configurations tuned by OHL versus other algorithms on multiple kernel regression. A linear combination of the ARD kernel and periodic kernel is used. Flow data from past 20 time-steps are used as autoregressive features. The feasibility sets for the hyperparameters are $[1.5*10^{-2},1.5*10^{-6}]$ for kernel scales, $[96/2,96*7]$ for periodicity in the periodic kernel, $[0.03,3]$ for regularization constant. The following hyperparameter tuning algorithms are tested:
\begin{itemize}
	\item The Online Hyperparameter Learning (OHL) proposed in this paper.
	\item Hyperparameter Optimization with Approximate Gradient (HOAG) \cite{pedregosa2016hyperparameter}, a state-of-art gradient-based hyperparameter optimization method.
	\item Random Search (RDS) \cite{bergstra2012random}.
	\item Grid Search and fixed hyperparameters (FIXED), a baseline. 
\end{itemize}

Our experimental procedure aims to mimic the application scenario of a traffic prediction engine. Note that both HOAG and RDS are offline tuning strategies. Therefore, to apply HOAG and RDS in a running environment, the traffic operators need to re-run these tuning strategies periodically, as described in the rolling protocol (Algorithm~\ref{rollingprotocol}) in the introduction. Hence, we set the hyperparameter optimization interval for HOAG and RDS as $n=96*7$, which corresponds to weekly model tuning. At each hyperparameter tuning step, the validation set $V_t$ consists of observations in the past one month, and these are given to the tuning algorithm. The tuning algorithms HOAG and RDS then perform backtesting on the validation data $V_t$. HOAG uses hyperparameter gradient information to guide the search for optimal hyperparameters on $V_t$~\cite{pedregosa2016hyperparameter}. RDS experiments with previously selected configuration and 50 additional random configurations on the validation dataset $V_t$; the one offering the best backtesting accuracy is used for the next period. 

RDS is simple to implement and often produces good hyperparameter tuning results. It is thus widely used in practice\cite{bergstra2012random}. After the model hyperparameters are selected for the weekly interval, the model is trained with a training set of $|S_t|=2880$ time-steps.

OHL is an online method, taking the streaming data and adaptively updating the hyperparameters. Therefore, it  does not require backtesting with the validation data $V_t$. The hyperparameter learning rate is set to $\eta=10^{-4}$ in OHL, and $m$ is set to 96. Therefore, OHL computes the hyperparameter gradient online and makes an adaptive update every 96 steps. We keep the training frequency and amount of training data the same for all methods.  

\subsubsection{Computational Efficiency Comparisons}

Figure~\ref{fig:runtime} gives the total time and the hyperparameter tuning time for the methods OHL, HOAG, and RDS, and for the 13 detectors. Overall, OHL is nearly $\mathbf{7\times}$ faster than HOAG and RDS. The speedup is mainly due to the faster hyperparameter tuning in OHL compared to HOAG and RDS. The average hyperparameter tuning times for OHL and HOAG are 3.2 and 120 minutes, respectively, indicating a tuning speedup of $37.5\times$. The time spent on hyperparameter selection in OHL is a small percentage of the total running time, and hence additional computational overheads are not introduced compared to predictions under FIXED hyperparameters. The dramatic speedup in OHL over the slow execution of HOAG and RDS is expected: although HOAG and RDS are both good hyperparameter tuning algorithms for I.I.D. setting, the rolling procedure for time-series prediction applies the tuning algorithms periodically according to the operation schedule (weekly in our experiments). Each run of the hyperparameter optimization algorithms requires backtesting on the validation set, which also in turn involves multiple parameter fitting steps corresponding to different hyperparameters. Even though HOAG uses gradient-based optimization, the algorithm searches for the optimal solution of hyperparameters on $V_t$ during each tuning stage. In comparison, OHL extracts the hyper-gradient knowledge adaptively from the streaming data, and making a single projected-gradient update to the hyperparameters before re-training the model. Further, the overall times (average across 13 detectors) for OHL and HOAG are 20 and 138 minutes, indicating a $6.9\times$ speedup.

\subsubsection{Prediction Accuracy}
\begin{figure*}[ht!]
	\centering
	\includegraphics[width=0.32\textwidth]{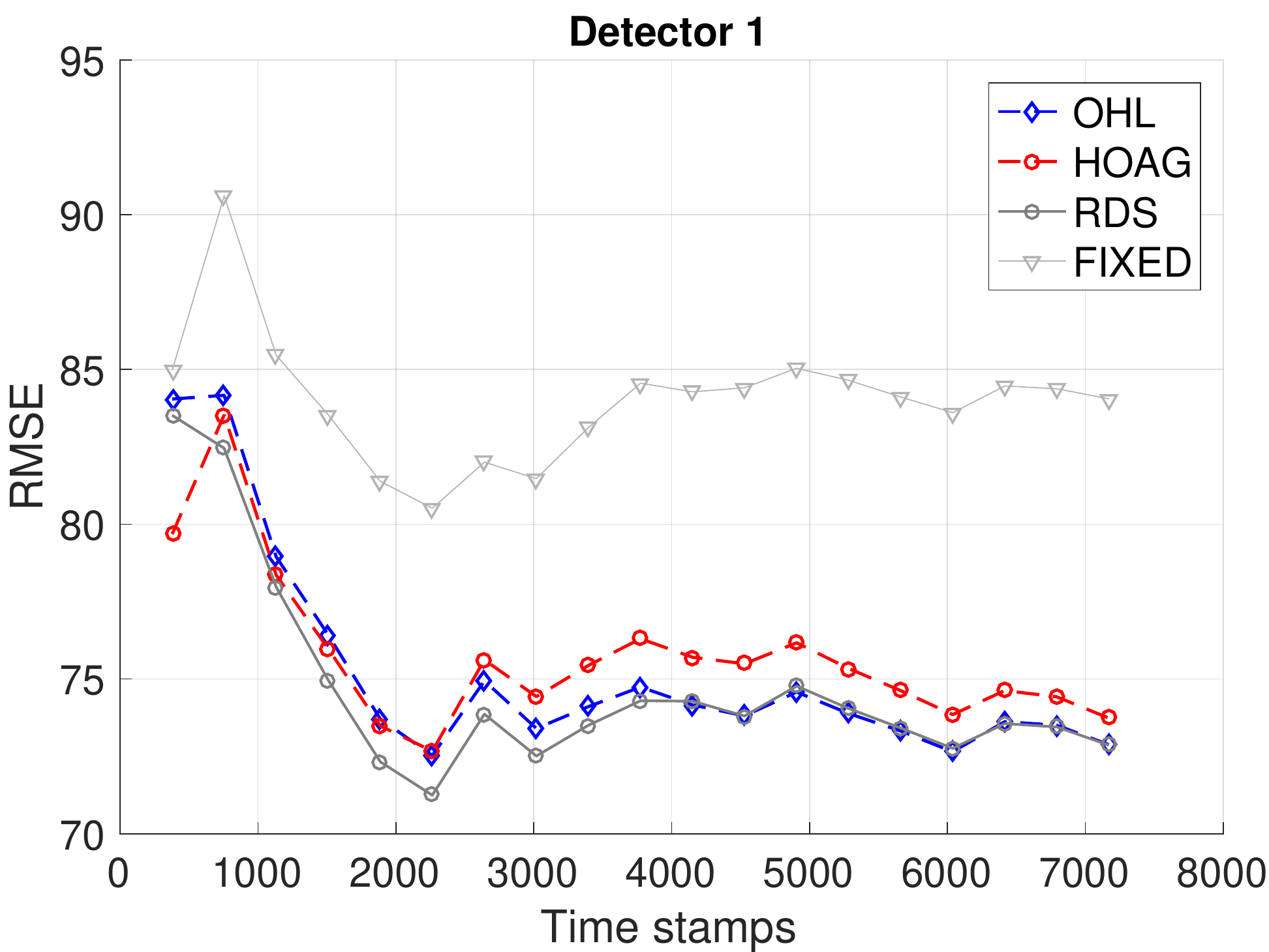}
	\includegraphics[width=0.32\textwidth]{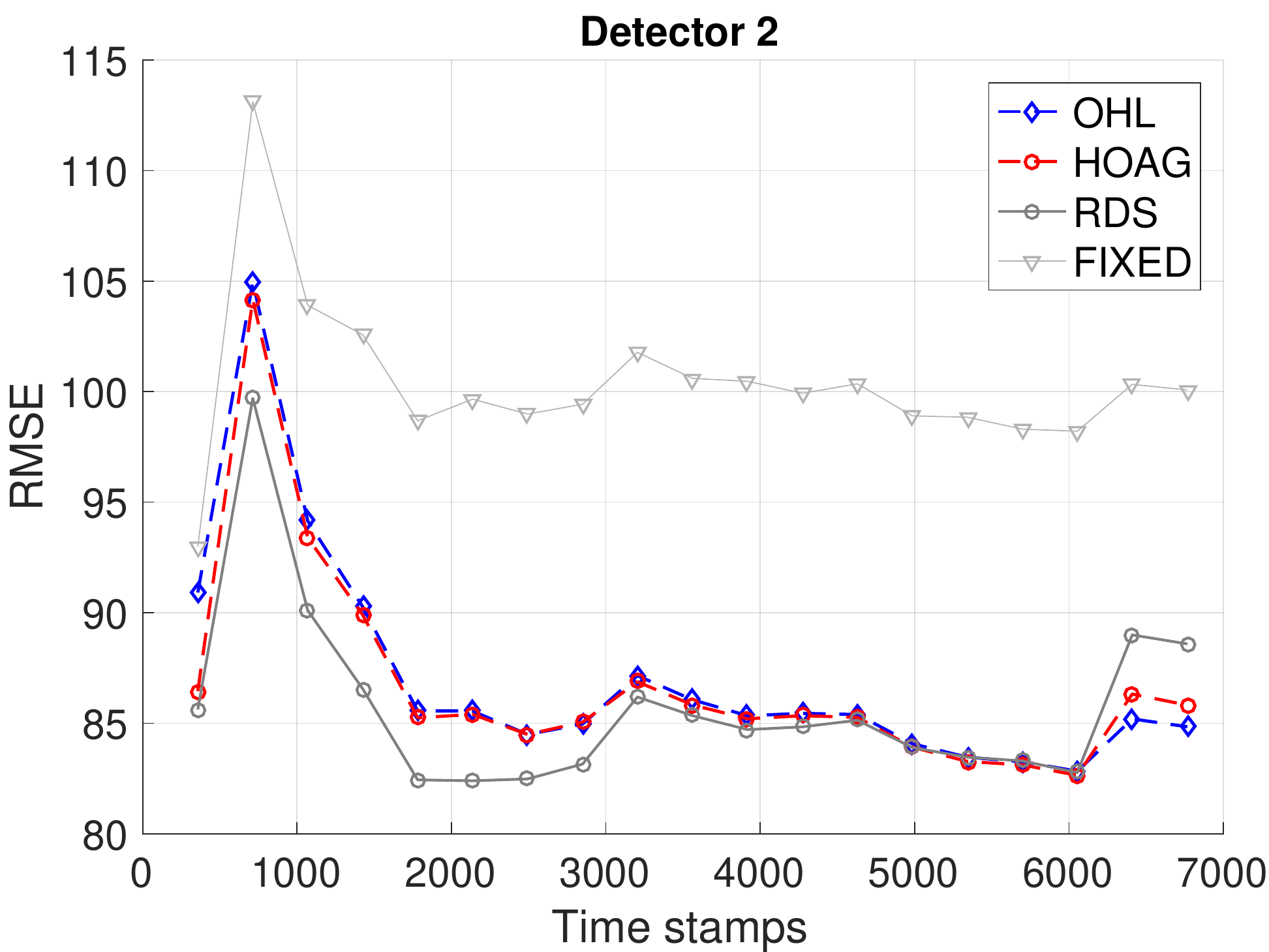}
	\includegraphics[width=0.32\textwidth]{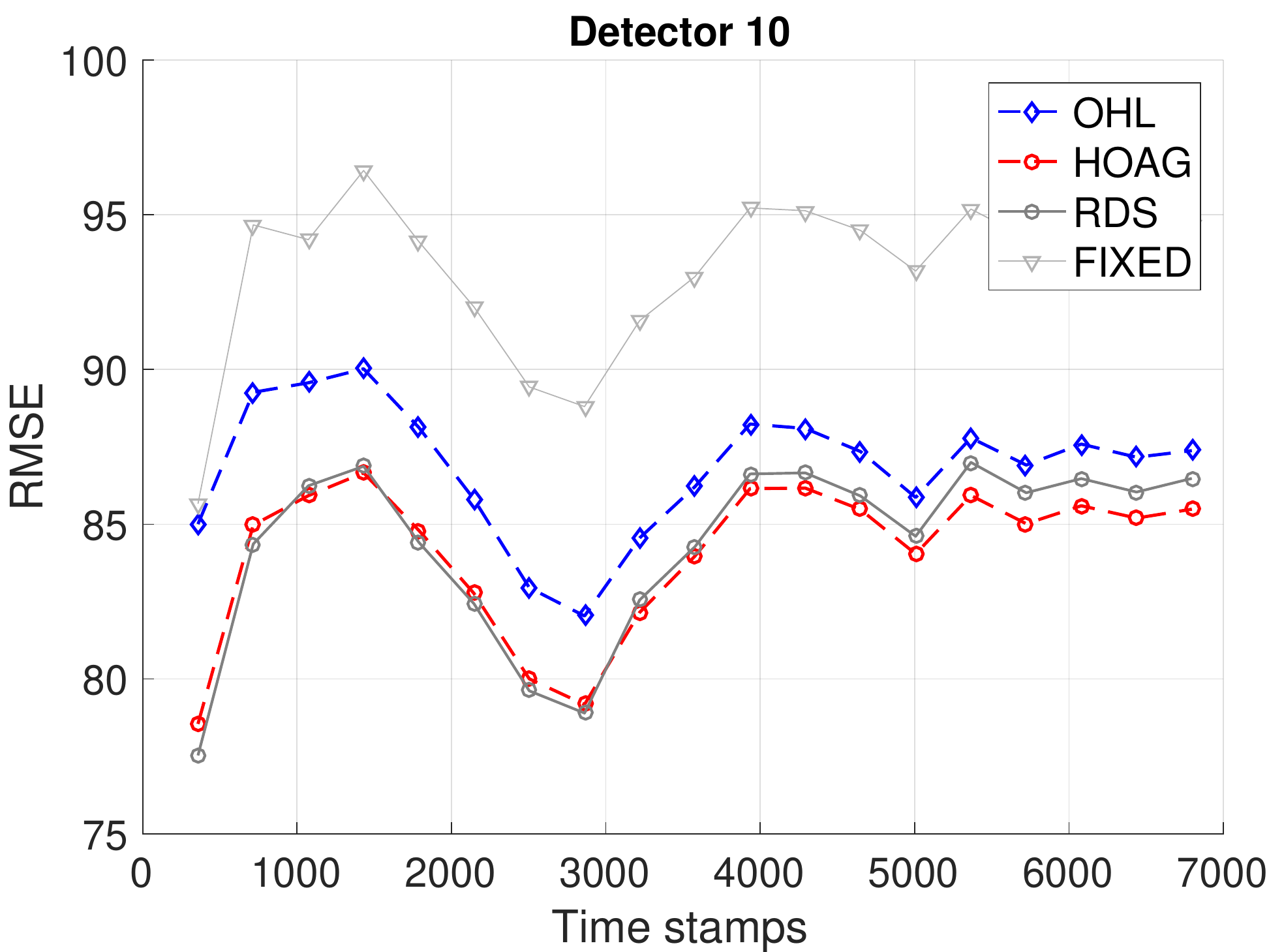}
	\includegraphics[width=0.32\textwidth]{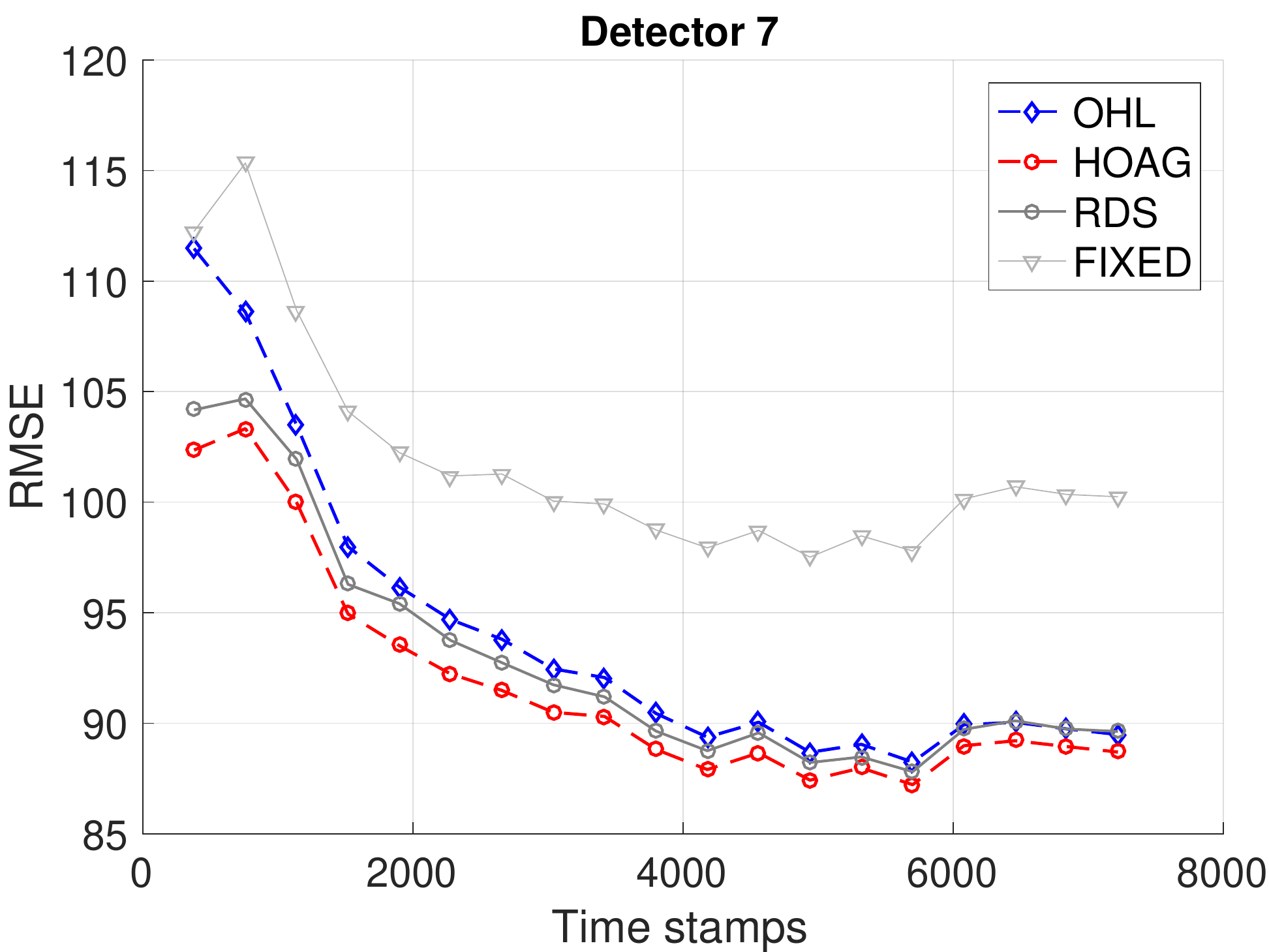}
	\includegraphics[width=0.32\textwidth]{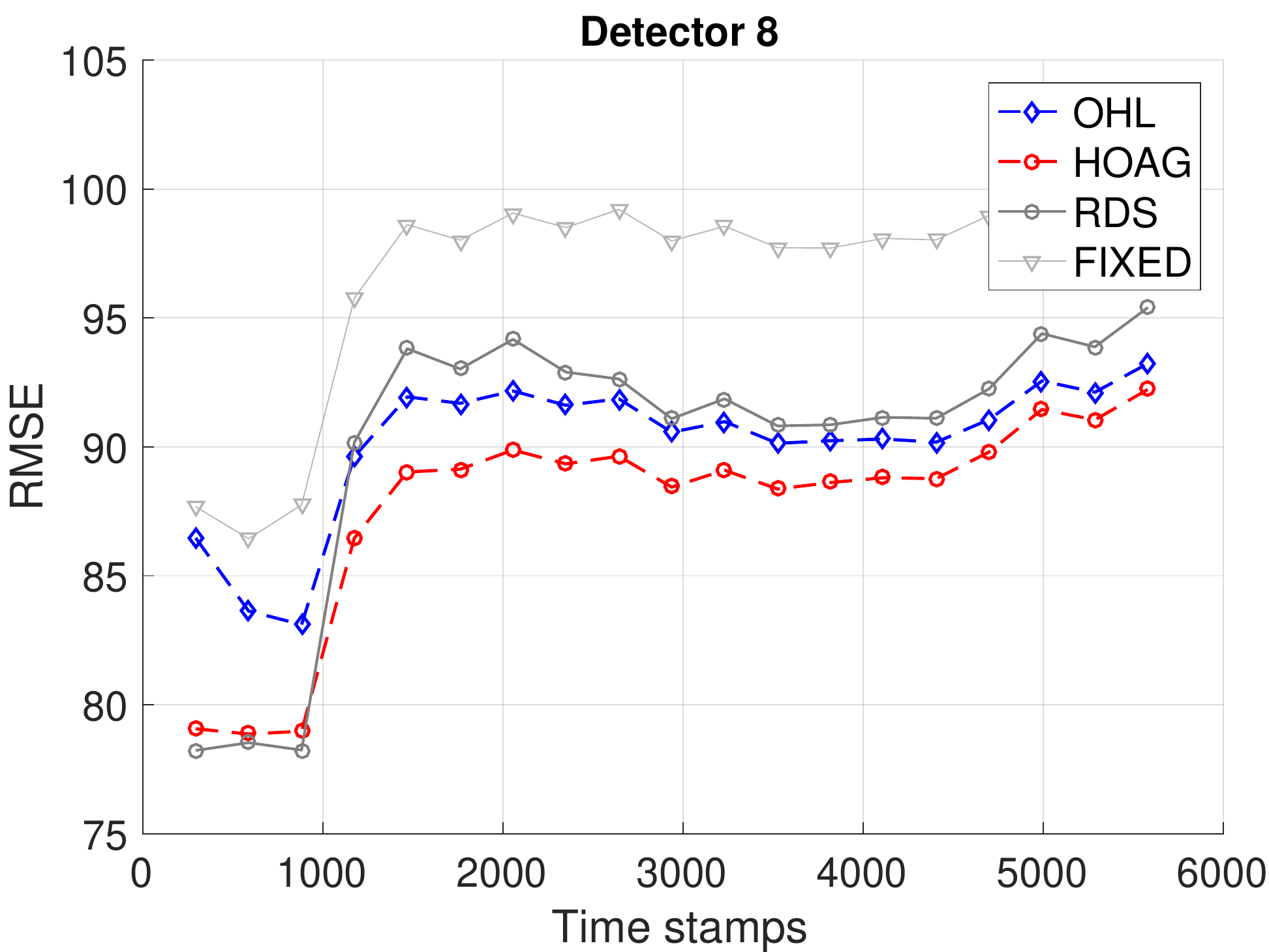}
	\includegraphics[width=0.32\textwidth]{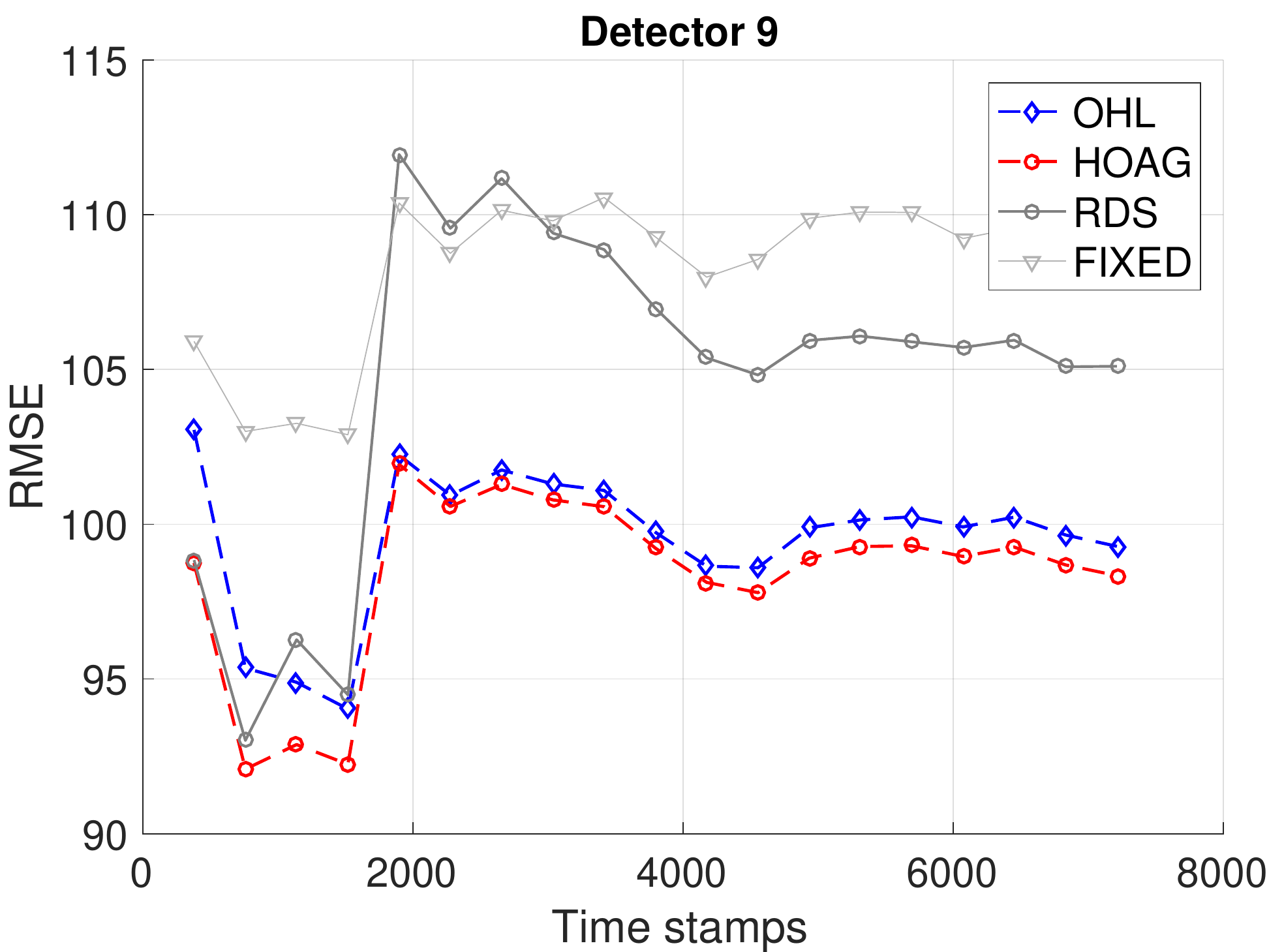}
	\includegraphics[width=0.32\textwidth]{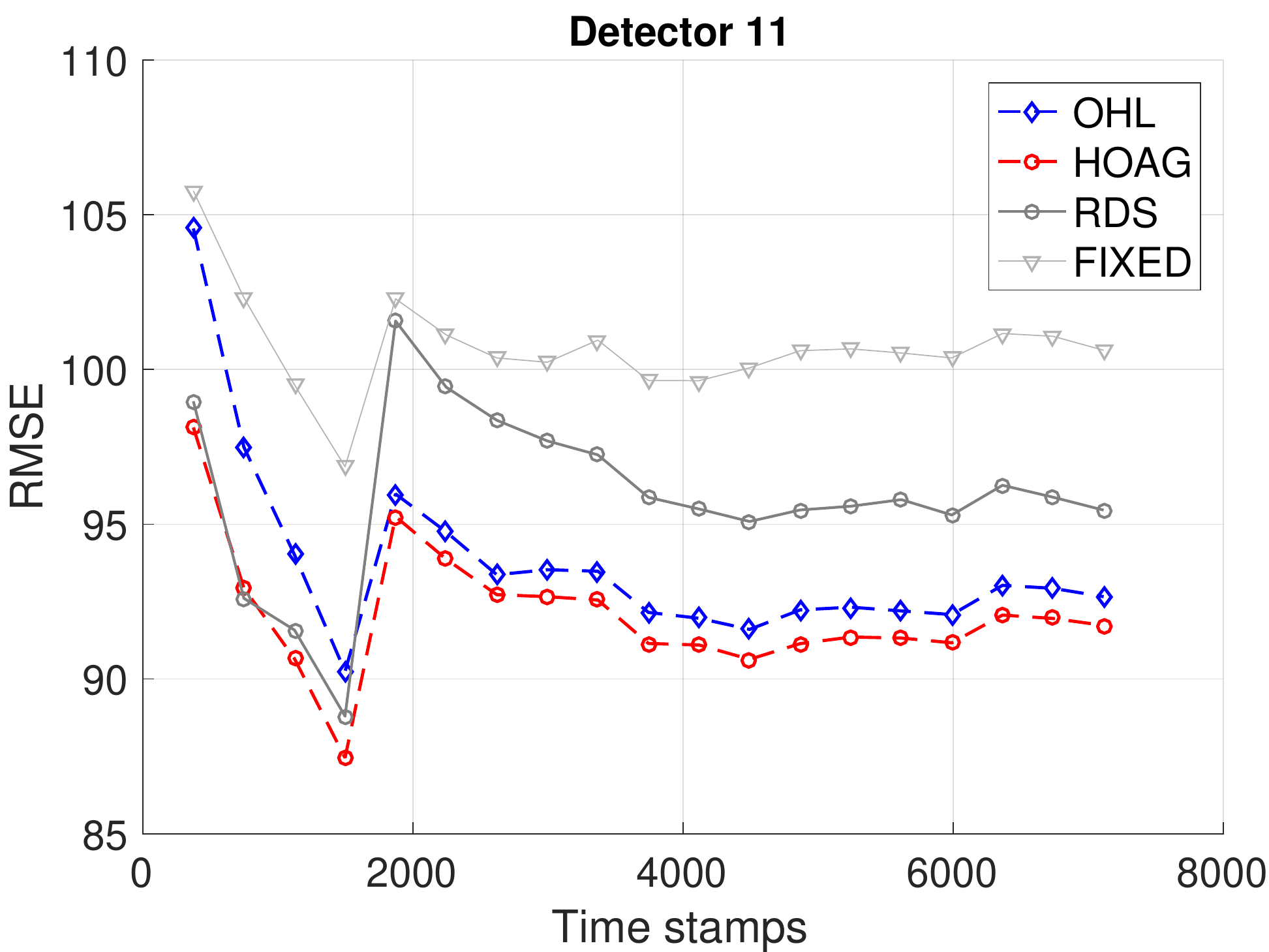}
	\includegraphics[width=0.32\textwidth]{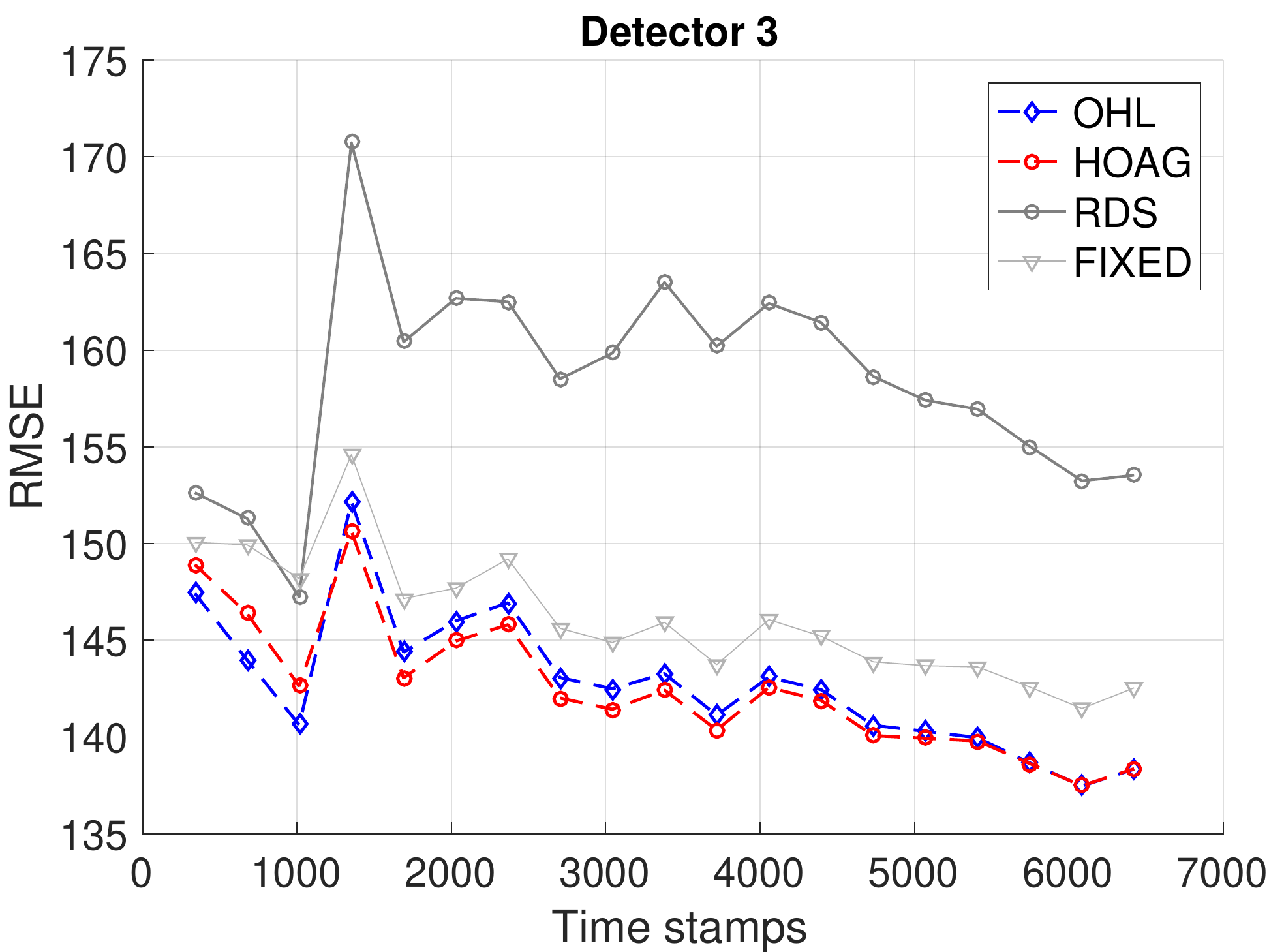}
	\includegraphics[width=0.32\textwidth]{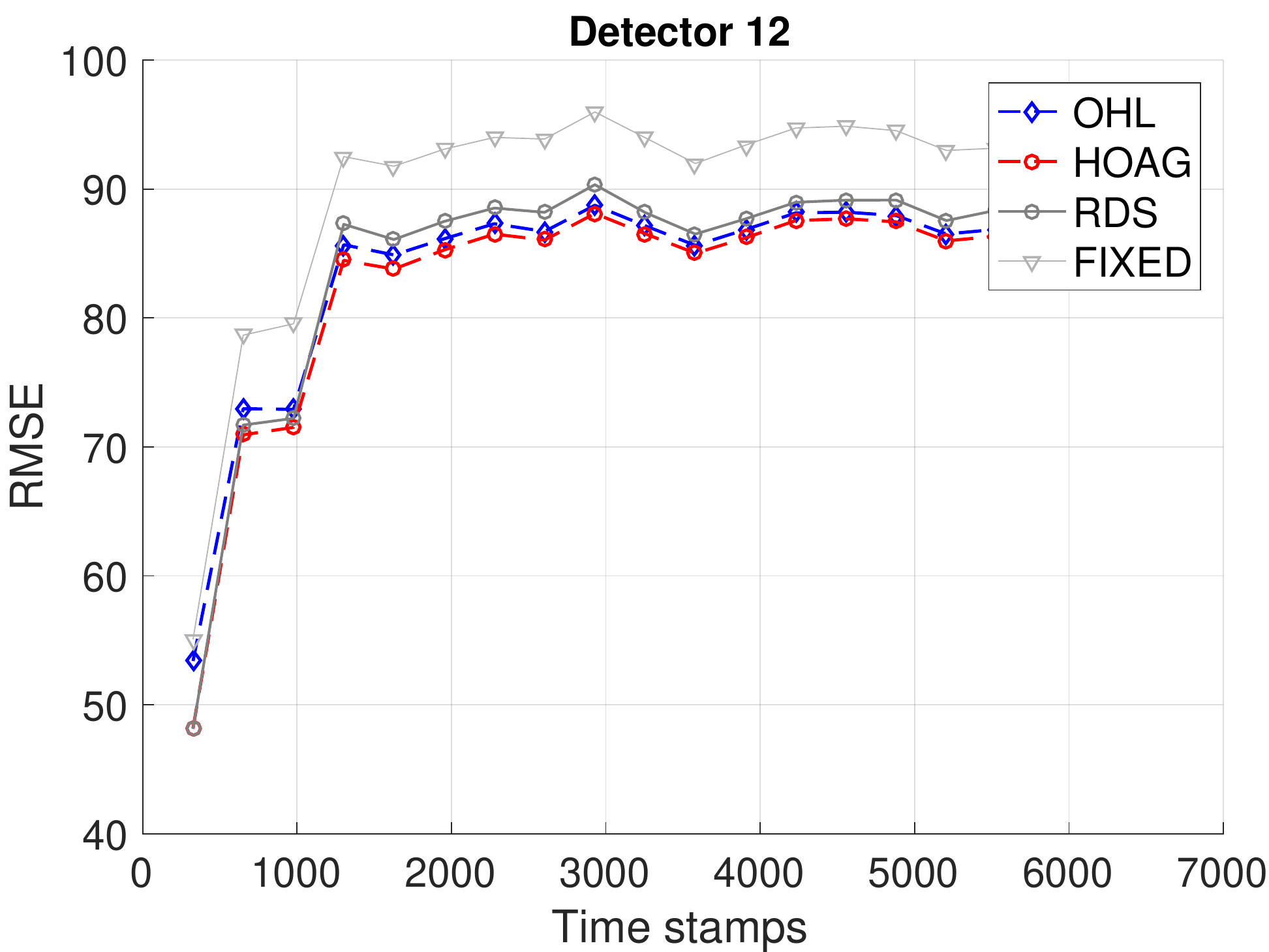}
	\includegraphics[width=0.32\textwidth]{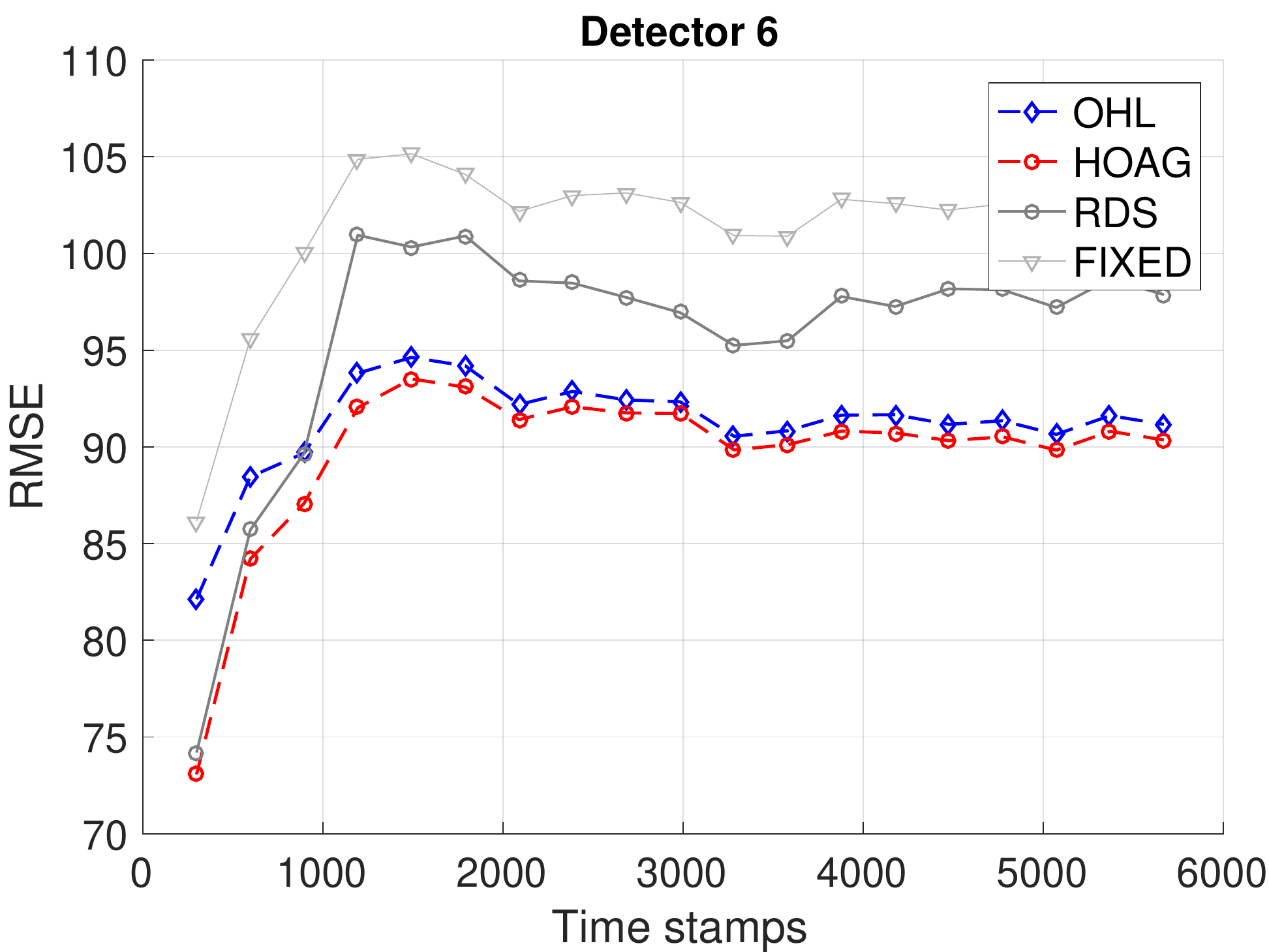}
	\includegraphics[width=0.32\textwidth]{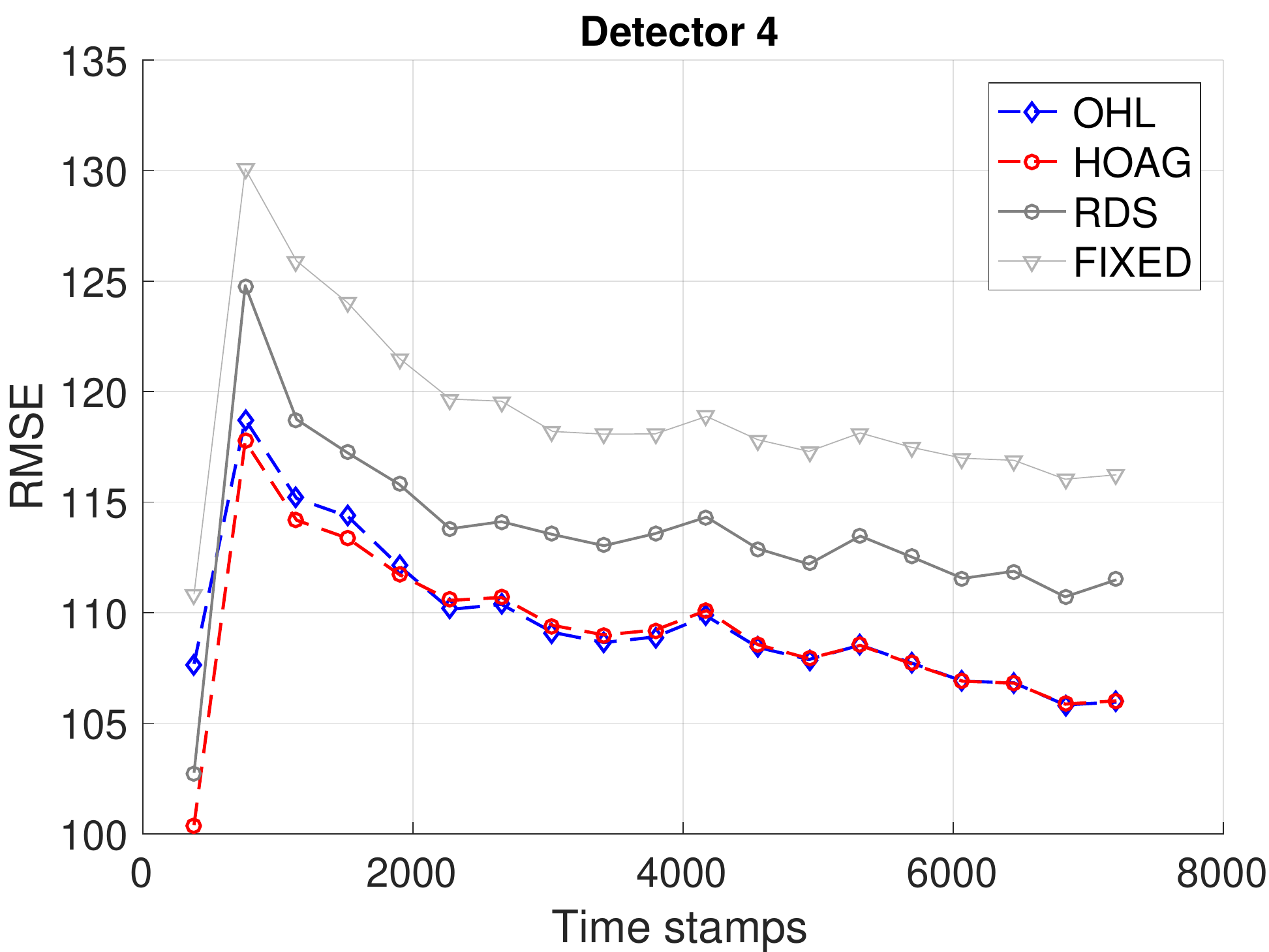}
	\includegraphics[width=0.32\textwidth]{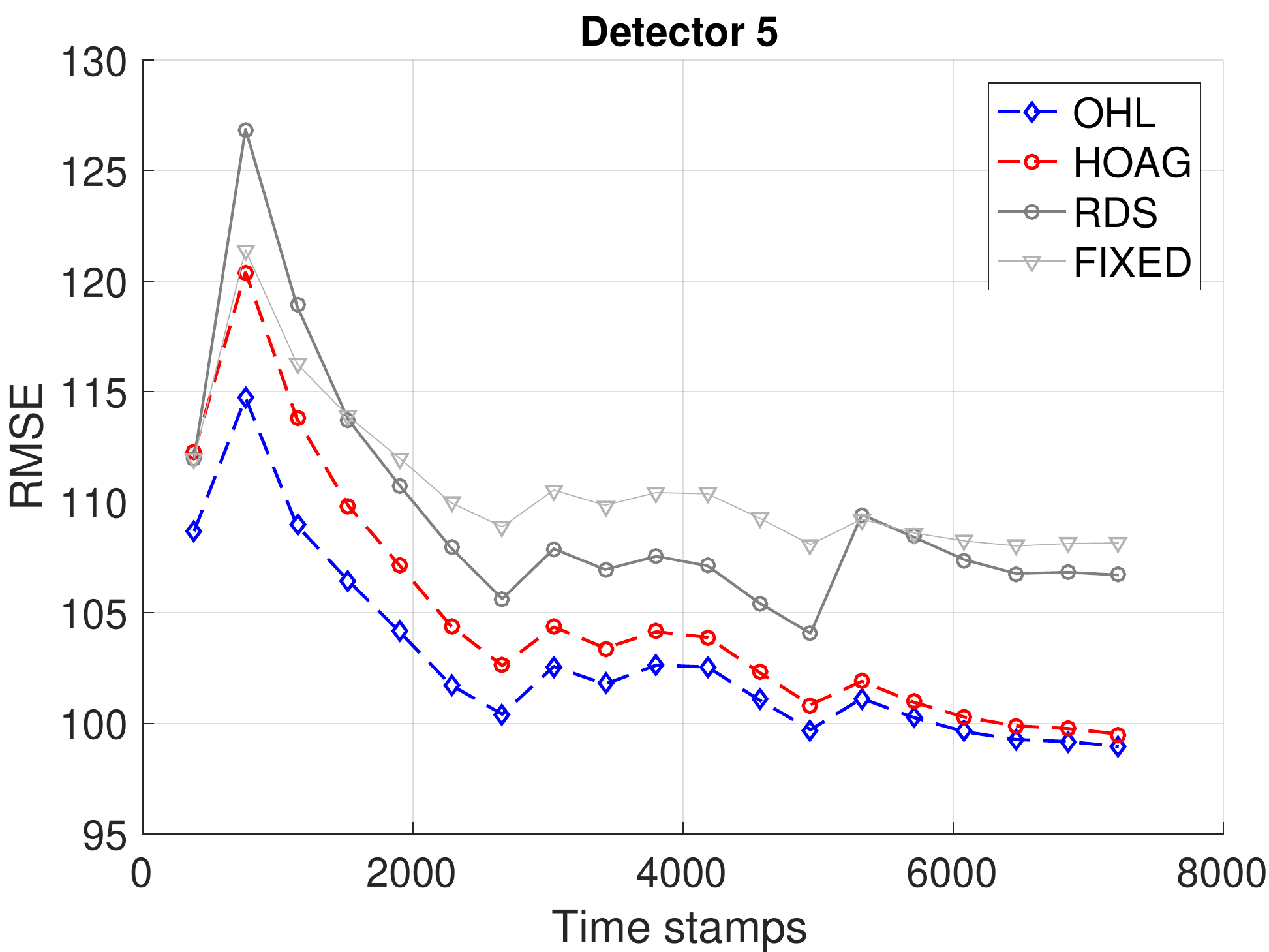}
	\includegraphics[width=0.32\textwidth]{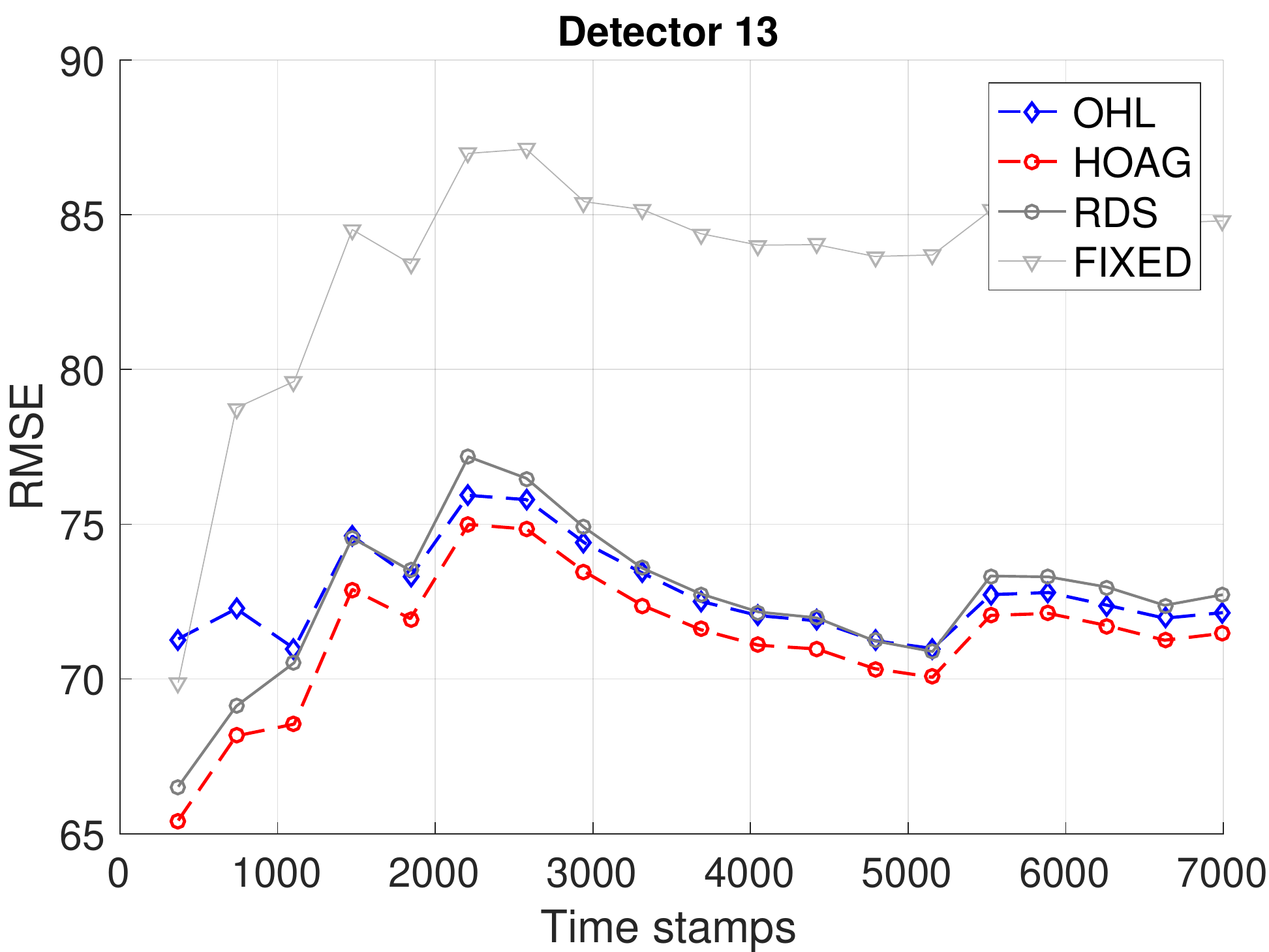}	
	\caption{\small{Prediction RMSE comparison for traffic flow data, OHL vs other hyperparameter tuning algorithms.}}	
\label{fig:accuracy}	
\end{figure*}

We use Root Mean Squared Error (RMSE) to measure the traffic flow prediction accuracy with different hyperparameter optimization algorithms. In order to examine the variation of RMSE over time, we also report the RMSE as a function of time-step $t$:
$$\RMSE(t):=\sqrt{\frac{1}{t}\sum^t_{\tau=0} (y_{\tau} - \hat{y}_{\tau})^2},$$
Thus, $\RMSE(t)$ summarizes the average prediction error from the start to time-step $t$. Due to space limit, we only show the evolution of $\RMSE(t)$ over the testing period for six detectors in Figure~\ref{fig:accuracy}. The X axis in Figure~\ref{fig:accuracy} is the prediction time-step and the Y axis corresponds to the RMSE up to that time-step. OHL, HOAG, and RDS have lower RMSE compared to the result with FIXED hyperparameters. This is an indication of suboptimal hyperparameters selected and then fixed by grid search. The prediction accuracy of OHL is similar to HOAG in most cases, and sometimes better. For example, on detector 1 and detector 2 (first two charts in the top row of Figure~\ref{fig:accuracy}), RDS has the lowest RMSE in the beginning phrase of testing, but OHL gradually improves and outperforms others over time. Meanwhile, OHL is also computationally the cheapest among the three hyperparameter tuning strategies. The similar prediction accuracy between HOAG and OHL in most cases are expected (Table~\ref{tab:accuracy-time}), since both use gradient-based optimization on hyperparameters. However, HOAG periodically applies the gradient-based iterations until convergence on the validation dataset, which makes the overall computation costly. In contrast, OHL achieves the same result with adaptive updates. On some detectors, RMSE of the model tuned by OHL algorithm is lower than HOAG - an offline gradient-based counter part. There are two reasons that can explain why an online HO algorithm performs better than an offline one. The optimal hyperparameters on validation set underperform in future data, suggesting that either there is \textit{overfitting by hyperparameters} or the data distribution is not stationary. On the contrary, OHL enables timely hyperparameter updates adapted to the latest observations.

\begin{table}[!h]
	\setlength{\tabcolsep}{2pt}
	\sisetup{ 
		table-number-alignment = center,
		table-format = 2.1,
		table-auto-round,
		zero-decimal-to-integer,
	}
	\centering
	\caption{RMSE percentage improvement relative to FIXED hyperparameters. Larger values are better. OHL achieves \emph{similar accuracy} to HOAG, and nearly $7\times$ faster (see Figure~\ref{fig:runtime}).}
	\begin{tabular}{|r|r|r|r|r|r|r|r|}
		\toprule
		& \multicolumn{7}{@{}c@{}|}{$\RMSE \text{ Improvements }$} \\  & \multicolumn{3}{@{}c@{}|}{4000 time-steps}  & & \multicolumn{3}{@{}c@{}|}{final} \\
		ID & OHL & HOAG & RDS & & OHL & HOAG & RDS  \\
		\hline 
		1 &12\%    &11\%   &12\% & 	& 13\% & 12\%  & 13\%   \\
		2 &15\%    &15\%   &15\% & & 15\%  & 14\% & 11\%  \\
		3 &2\%  & 2\%   &-11\% &  & 3\%  & 3\%  & -8\%  \\
		4 &8\%   &7\%    &4\%	 & & 9\% & 9\%  & 4\%  \\
		5  &8\%   &6\%   &3\% & & 9\% & 8\%  & 1\%\\
		6 &11\%    &11\%   &5\% & & 11\%  & 11\%  & 4\% \\
		7 &9\%  &10\%   &9\% & &	11\% & 12\%  & 11\% \\
		8 &8\% & 9\% &7\% & & 8\% & 9\%  & 6\% \\
		9 &8\% &  9\%   &2\% & &	9\%  & 10\% & 3\% \\
		10 &7\% &9\%   &9\% & &	8\%  & 10\%  & 9\%	 \\
		11 &8\%  & 8\% & 4\% & & 8\% & 9\% & 5\% \\
		12 &7\%  & 7\% & 6\% & & 7\% & 8\%  &5\% \\
		13 & 14\%  & 15\% &14\%	&  & 15\% &  16\%  & 14\%	\\
		\bottomrule
	\end{tabular}
	\label{tab:accuracy-time}
\end{table}

\section{Conclusions}
Motivated by the need for hyperparameter optimization in traffic time series prediction, we proposed the OHL algorithm and applied it on Multiple-Kernel Ridge Regression. The proposed OHL algorithm achieves optimal local regret. In the traffic flow prediction experiments, OHL is nearly $7\times$ faster than other rolling hyperparameter tuning methods, while achieving similar prediction accuracy. In addition, we observed a consistent improvement in accuracy compared to predictions produced with static hyperparameters. 

There are possible extensions to this work: efficient online hyper-gradient approximation methods for a general class of models can expand the application scope of OHL. One direction of improvement is combining our OHL algorithm with the reverse-mode and forward-mode computation of hyper-gradients~\cite{franceschi17a}. 

\section*{Acknowledgements}
This work is supported by the US National Science Foundation
grant ACI-1253881 and a Penn State College of Engineering seed grant.

\bibliography{refs}
\bibliographystyle{unsrt}

\end{document}